\documentclass[conference]{IEEEtran}
\usepackage{times}
\usepackage[numbers]{natbib}
\usepackage{multicol}

\usepackage{epsfig}
\usepackage{graphicx}
\usepackage{fancyhdr}
\usepackage{color}
\usepackage{amsmath,amssymb} 
\usepackage[english]{babel}
\usepackage[utf8]{inputenc}
\usepackage{caption}
\usepackage{subcaption}
\usepackage{siunitx}
\usepackage{enumerate}
\usepackage{amsthm}
\usepackage{ragged2e}
\newtheorem{theorem}{Theorem}
\usepackage[font={small}]{caption}
\usepackage{multirow}
\usepackage{footnote}
\usepackage{algpseudocode}
\usepackage{algorithm}
\usepackage[utf8]{inputenc}
\usepackage{xcolor}

\newcommand\Mark[1]{\textsuperscript{#1}}

\begin{document}



\title{Pixel-variant Local Homography for Fisheye Stereo Rectification Minimizing Resampling Distortion\\[.75ex] 
  {\normalfont\large 
    Dingfu Zhou\Mark{1}, Yuchao Dai \Mark{1} and Hongdong Li\Mark{1,2}%
  }\\[-1.5ex]
}

\author{
    \IEEEauthorblockA{%
        \Mark{1} Research School of Engineering, The \\Australian National University (ANU), Australia%
    }
    \and
    \IEEEauthorblockA{%
        \Mark{2}Australia Centre of Excellence for \\Robotic Vision (ACRV), Australia%
    }
}
%
\maketitle
\begin{abstract}
Large field-of-view fisheye lens cameras have attracted more and more researchers' attention in the field of robotics. However, there does not exist a convenient off-the-shelf stereo rectification approach which can be applied directly to fisheye stereo rig. One obvious drawback of existing methods is that the resampling distortion (which is defined as the loss of pixels due to under-sampling and the creation of new pixels due to over-sampling during rectification process) is severe if we want to obtain a rectification with epipolar line (not epipolar circle) constraint. To overcome this weakness, we propose a novel pixel-wise local homography technique for stereo rectification. First, we prove that there indeed exist enough degrees of freedom to apply pixel-wise local homography for stereo rectification. Then we present a method to exploit these freedoms and the solution via an optimization framework. Finally, the robustness and effectiveness of the proposed method have been verified on real fisheye lens images. The rectification results show that the proposed approach can effectively reduce the resampling distortion in comparison with existing methods while satisfying the epipolar line constraint. By employing the proposed method, dense stereo matching and 3D reconstruction for fisheye lens camera become as easy as perspective lens cameras.  
\end{abstract}
\IEEEpeerreviewmaketitle
\section{Introduction}
Fisheye lens cameras have been widely used in the fields of computer vision \cite{caruso2015large}, robotics \cite{zhangbenefit} and augmented reality/virtual reality, (e.g., Project Tango \cite{hane2014real}), due to its large field of view. It is natural to use a pair of fisheye lens images, or even motion sequence for Structure from Motion (SfM) or dense 3D reconstruction, where stereo rectification is often a prerequisite. Although stereo-rectification has been thought to be a well solved problem, surprisingly a convenient off-the-shelf solution is not available for fisheye lens images. An opportunistic and commonly used solution is first to correct the fisheye image to perspective image, and then do standard perspective rectification \cite{zhang2015line}. The drawback of this approach (as shown in subfig.(\ref{subfig:rec_per_building})) is large rectification resampling distortion which will hinder the post-processing e.g., sparse features tracking or dense stereo matching.

The large distortion is caused by two aspects: 1) undistorting fisheye image to perspective will generate resampling distortion; 2) rectifying two perspective images will also generate resampling distortion. In order to avoid resampling distortion during the undistortion step, many researchers proposed to rectify the fisheye image on a spherical surface rather than a plane \cite{heller2009stereographic,fujiki2007epipolar,geyer2003conformal}. By doing this, although the resampling distortion has been reduced, the dense stereo matching is difficult because epipolar lines become curves in spherical coordinates. 

Recently, plane-sweeping technique \cite{gallup2007real} has been widely applied to estimate dense depth for perspective cameras. By using plane-sweeping, the stereo rectification step is not necessarily required. Furthermore, H\"{a}ne \emph{et al.} transferred this technique for fisheye images in \cite{hane2014real}. However, before applying plane-sweeping, the tangential and radial distortions of fisheye images should be removed. Therefore, this kind of methods cannot work properly when camera's field-of-view is relative large, because the resampling distortion will be serious in the undistortion step.  

\begin{figure}[t!]
\centering
\begin{subfigure}[t]{0.475\textwidth}
\centering
\includegraphics[width=0.95\textwidth,height = 0.275\textwidth]{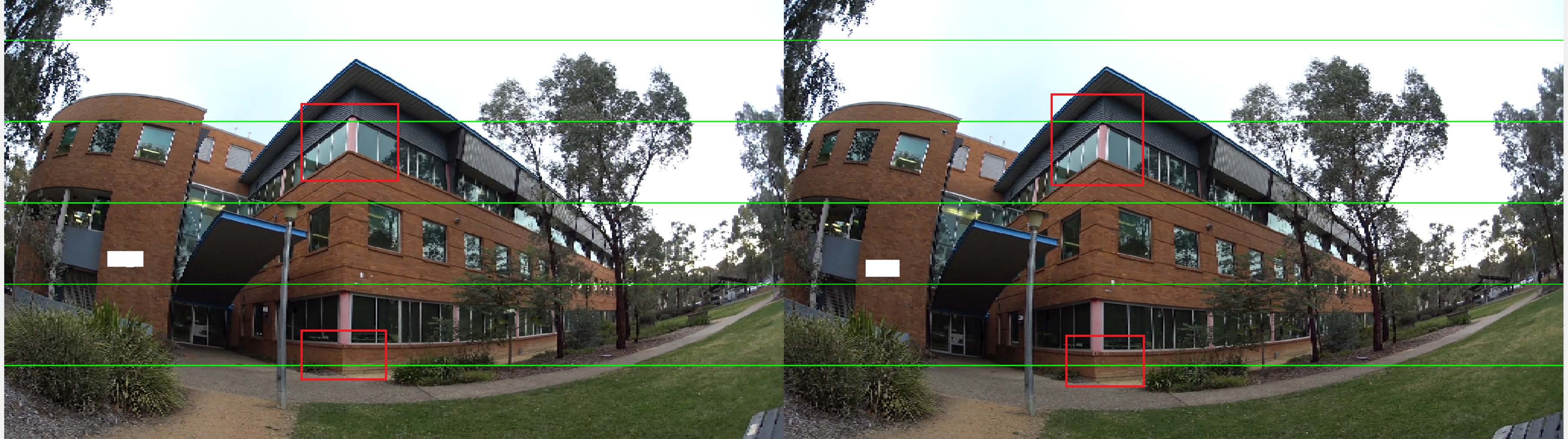}
 \caption{Fisheye image pair.}
\label{subfig:Left_Right_Image_building}
\end{subfigure}
\\
\begin{subfigure}[t]{0.475\textwidth}
\centering
\includegraphics[width=0.95\textwidth,height = 0.275\textwidth]{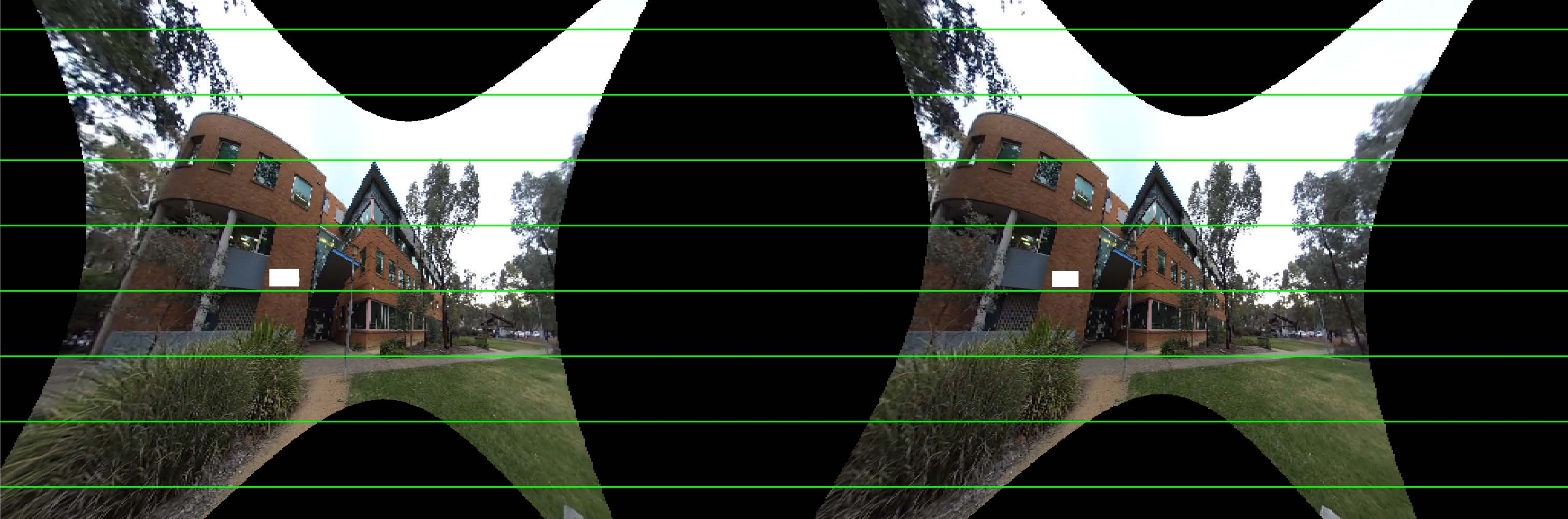}
\caption{Rectification results by a conventional method \cite{bouguet2004camera}.}
\label{subfig:rec_per_building}
\end{subfigure}
\\
\begin{subfigure}[t]{0.475\textwidth}
\centering
\includegraphics[width=0.95\textwidth,height = 0.275\textwidth]{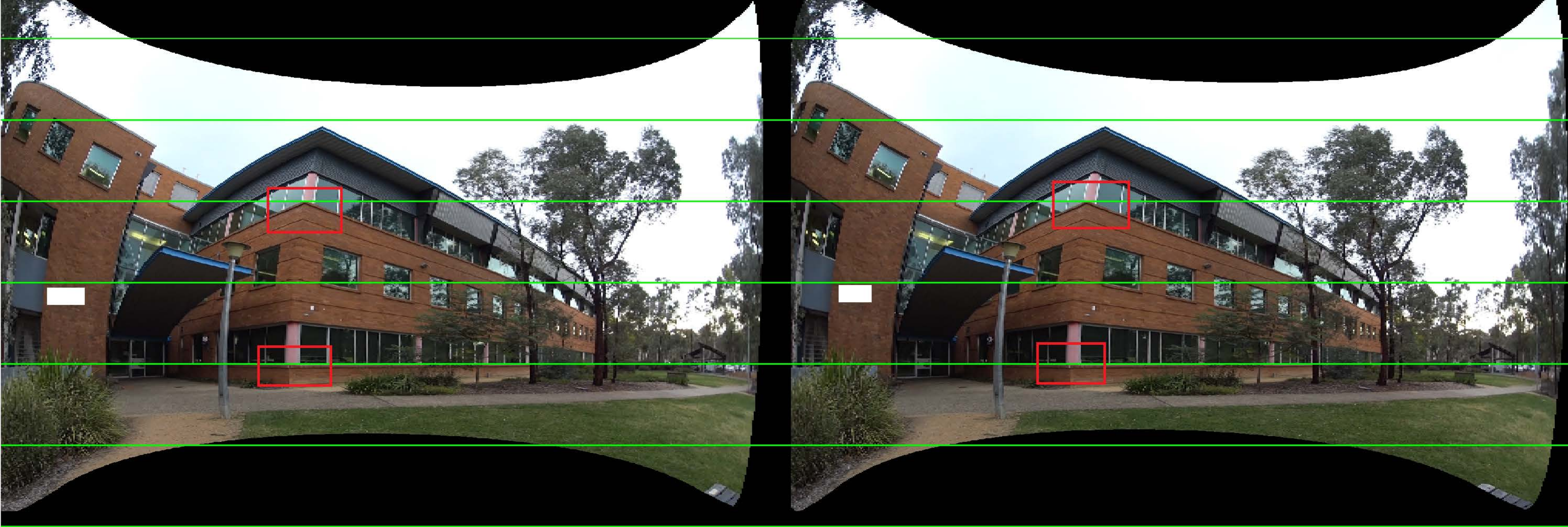}
\caption{Rectification results by the proposed method.}
\label{subfig:pro_fish_bilding}
\end{subfigure}
\caption{Rectification results with different methods. Compared with the conventional method (rectification after removing fisheye distortion), the proposed method effectively minimizes the resamlping distortion while keeping the field of view as original images. Rectangles are used to highlight that epipolar line constraints are satisfied after rectification.} 
\label{fig:rect_fish_building} 
\end{figure}
\vspace{-0.1cm}
To overcome these disadvantages, we propose to find a direct transformation to map the original fisheye image to a rectified one. By doing this, the resampling distortion caused in undistortion step can be avoided. Furthermore, we do not enforce the rectified images follow ideal perspective imaging model. Instead, they can remain their ``fisheye property'' as long as corresponding epipolar lines are parallel and properly arranged. To obtain these properties, we propose a pixel-wise local homography for rectifying fisheye images. Generally speaking, we use different rectifying transformations at different pixel location.

The first contribution of this paper is that: we have proved that there exist enough degrees of freedom to choose the rectifying transformation matrices for stereo rectification. In contrast to applying a single global homography matrix pair for the entire image \cite{gluckman2001rectifying,hartley1999theory,fusiello2008quasi,monasse2010three}, we present a pixel-wise local homography for stereo rectification. Typically, different transformation homographies are applied at different image pixels. Consequently, the rectification resampling distortion has been dramatically reduced as multiple rectifying transformations are considered. 

The second contribution of this paper is to provide a tailored stereo rectification algorithm for fisheye lens images by exploiting the extra degrees of freedom in choosing the transformation matrices. After rectification, the epipolar lines are aligned with the image scan rows, and more importantly the resampling distortion has been minimized. The rectified stereo image pair can be used directly to compute dense stereo matching and 3D reconstruction for further applications in the field of computer vision and robotics. Subfig.(\ref{subfig:pro_fish_bilding}) displays an example of rectification result by using the proposed method. The local information of the original image has been mostly preserved in the rectified images while the epipolar line constraints are also satisfied.

\subsection{Related works}
Many approaches have been proposed to reduce resampling distortion for perspective images \cite{1603.09462,gluckman2001rectifying,fusiello2008quasi,pollefeys1999simple,palander2008epipolar}. Most of them try to find an optimal global transformation matrix (e.g., homography matrix \cite{gluckman2001rectifying,hartley1999theory,loop1999computing,monasse2010three}) to achieve minimum resampling distortion. 

Besides perspective images, many approaches also have been proposed for fisheye lens cameras. Heller \emph{et al.} proposed a general technique for rectifying omni-directional stereo image by mapping the epipolar curves (the epipolar line in perspective image becomes curve in fisheye image) onto circles to reduce the resampling distortion \cite{heller2009stereographic}. In addition, stereographic projection is simply applied for stereo rectification. However, their technique can only obtain epipolar curves rather than epipolar lines. Zafer \emph{et al.} proposed to rectify the fisheye image on spherical coordinates to reduce rectification distortion \cite{arican2007dense}. Although the resampling distortion has been reduced in their methods, traditional dense matching methods cannot be applied directly on rectified images because the correspondences are on circles rather than scan-lines. For parabolic catadioptric cameras, Geyer \emph{et al.} proposed a so called conformal rectification approach by using their conformal properties  \cite{geyer2003conformal}. Unfortunately, it only works for parabolic catadioptric cameras. Abraham \emph{et al.} \cite{abraham2005fish} proposed to keep the ``fisheye property'' in the rectified image to reduce the resampling distortion. However, they just simply applied several existed camera models for the image rectification and didn't try to find an optimal rectification model with minimum resampling distortion. 
 
To remind the readers, the remaining parts of this paper are organized as follows: we first introduce the pixel-wise local homography in Section \ref{Sec:Problem_Statement} and then formulate the stereo rectification as an energy minimization problem. The implementation of stereo rectification and parameters optimization are given in Section \ref{sec:SR_solving}. Main rectification steps and 3D reconstruction are described in Section \ref{sec:3D_reconstruction} and the experimental results for rectification and 3D reconstruction are shown in Section \ref{sec:ExperimentalResults}. Finally, our paper ends with a short conclusion and some future works.  

\section {Pixel-wise Local Homography \label{Sec:Problem_Statement}}
Traditionally, stereo rectification is solved as finding a single {\em global homography transformation} which is applied {\em globally and uniformly} for the whole input image. In other words, the task of stereo rectification is represented as finding a pair of rectifying homography that best aligns the epipolar lines in each of the stereo image pairs (e.g.,  \cite{1603.09462,gluckman2001rectifying,fusiello2008quasi,pollefeys1999simple,palander2008epipolar}). Let's denote $(\mathbf{H}_1, \mathbf{H}_2)$ as the desired rectifying transformation pair, then the epipolar line constraint can be expressed as
\begin{equation}
(\mathbf{H}_{2}\mathbf{x}_{2})^{T}[\mathbf{e}]_\times \mathbf{H}_{1}\mathbf{x}_{1} = 0,
\label{Eq:homographyrectificaiton}
\end{equation}
where $\mathbf{x}_{1} = (u_{1}, v_{1}, 1)^T$ and $\mathbf{x}_{2} = (u_{2}, v_{2},1)^T$ are the feature correspondence between image $\mathbf{I}_1$ and $\mathbf{I}_2$. The fundamental matrix corresponds to the rectified image pair has a special form (up to a scale factor) as $\mathbf{F} = [\mathbf{e}]_\times$, where $\mathbf{e} = (1,0,0)^T$.  Based on epipolar geometry, rectifying transformation pair $(\mathbf{H}_1, \mathbf{H}_2)$ can be obtained from fundamental matrix (uncalibrated case) or essential matrix (calibrated case). Various approaches \cite{gluckman2001rectifying,hartley1999theory,loop1999computing,monasse2010three,1603.09462} have been proposed to find the optimal pair $(\mathbf{H}_1,\mathbf{H}_2)$ for minimizing the rectification resampling distortion.

\subsection{General homography transformation}
On contrast to the traditional approaches of using a single pixel-wise-uniform global homography, we propose to use pixel-variant local homographies here. In the derivation below we will first prove its possibility and then show its benefits of providing extra freedoms for choosing rectifying transformations, such as reducing the rectifying resampling distortion, etc. 

For convenience, the image pixel $\mathbf{x} = (u,v,1)^T$ is converted to its corresponding 3D bearing vector $\mathbf{b}$ (an unit 3D ray vector from the camera center to 3D point) using camera's intrinsic parameters. We denote $\mathbf{b} = \mathbf{\Phi} (\mathbf{x})$ as the mapping from 3D bearing vector to 2D image point and $\mathbf{x} = \mathbf{\Phi}^{-1}(\mathbf{b})$ as its inverse mapping. Both $\mathbf{\Phi}(.)$ and $\mathbf{\Phi}^{-1}(.)$ can be obtained by camera calibration \cite{scaramuzza2006toolbox}.

The epipolar constraint expressed with bearing vectors can be expressed as
\begin{equation}
\label{Eq:epipolar_constriant_bearingvector}
(\mathbf{\Phi}_2(\mathbf{x}_{2}))^{T} \mathbf{E} \mathbf{\Phi}_1(\mathbf{x}_{1}) = 0,
\end{equation}
where $\mathbf{\Phi}_1(\mathbf{x}_{1}$) and $\mathbf{\Phi}_2(\mathbf{x}_{2})$ are corresponding bearing vectors of correspondence $\mathbf{x}_{1}$ and $\mathbf{x}_{2}$. 

For convenience, the double rotation matrices parametric strategy \cite{yang2014optimal} is employed here to represent the essential matrix as $\mathbf{E} = \mathbf{R}_2 [\mathbf{e}]_{\times} \mathbf{R}_1^T$. More specifically, we set the first camera’s center as the origin and the second camera’s center as $[1, 0, 0]$ on the X-axis. Here, $\mathbf{R}_{1}$ and $\mathbf{R}_{2}$ are used to denote the absolute orientation of the first and second cameras relative to the world frame. Obviously, using two absolute rotations $(\mathbf{R}_{1}, \mathbf{R}_{2})$ to represent $\mathbf{E}$ is an over-parametrization, because $\mathbf{E}$ has only 5 degree of freedoms (dofs). However, under the above camera setup, any rotation about $X$-axis (i.e. the axis joining the two camera centers) applied to both cameras will leave $\mathbf{E}$ invariant. Due to this property, the essential matrix with our double rotation matrices parametric strategy also has 5 dofs. Substituting the essential matrix $\mathbf{E}$ in Eq.\eqref{Eq:epipolar_constriant_bearingvector}, we have
\begin{equation}
\label{Eq:epipolar_constriant_2R}
(\mathbf{R}_2^T \mathbf{\Phi}_2(\mathbf{x}_{2}))^T [\mathbf{e}]_{\times} (\mathbf{R}_1^T \mathbf{\Phi}_1(\mathbf{x}_{1})) = 0.
\end{equation}

However, the representation of Eq.\eqref{Eq:epipolar_constriant_2R} is not unique and there exist extra freedoms for the epipolar constraint, which to some extent has been neglected in the existing stereo rectification methods. Here, we prove that there exist a group of matrix pairs which satisfy $(\mathbf{A}')^{T} [\mathbf{e}]_{\times} \mathbf{A} = [\mathbf{e}]_{\times}$ and will not change the epipolar constraint and all these matrix pairs ($\mathbf{A}, \mathbf{A}'$) can be obtained from Theorem \ref{theo:theorem_1}.
\begin{theorem}
\label{theo:theorem_1}
There exists infinite number of transformation pairs $(\mathbf{A}, \mathbf{A}')$ satisfying $(\mathbf{A}')^{T} [\mathbf{e}]_{\times} \mathbf{A} = [\mathbf{e}]_{\times}$, which are given by\\
$\mathbf{A} = \begin{pmatrix}
a_{11}                & (\mathbf{A}_1)_{1\times2}  \\
\mathbf{0} & (\mathbf{A}_2)_{2\times2}  \\
\end{pmatrix}$ \text{and} 
$\mathbf{A}' = \begin{pmatrix}
a'_{11}                & (\mathbf{A}'_1)_{1\times2} \\
\mathbf{0}  & \lambda (\mathbf{A}_2)_{2\times2}\\
\end{pmatrix},$ 
where $\lambda \neq 0$ and $\mathbf{A}$, $\mathbf{A}'$ are any non-singular matrices.
\end{theorem}
\begin{proof}
Assuming any two transformation matrices as $\mathbf{A} = (a_{ij})$ and $\mathbf{A}' = (a'_{ij})$ ($i,j = 1,2,3$) and they satisfy 
\begin{equation}
(\mathbf{A}')^{T} [\mathbf{e}]_{\times} \mathbf{A} = [\mathbf{e}]_{\times}.
\label{eq:EpipolarConstraint}
\end{equation}
Substituting $\mathbf{A}$ and $\mathbf{A}'$ into Eq.(\ref{eq:EpipolarConstraint}), we have 
\begin{equation} 
\label{eq::equation_1}
\left\{
\begin{aligned}
a_{22}a'_{32} - a_{32}a'_{22} = 0, & ~a_{23}a'_{33} - a_{33}a'_{23} = 0 \\
a_{22}a'_{33} - a_{32}a'_{23} = 1, & ~a_{23}a'_{32} - a_{33}a'_{22} = -1 \\
\end{aligned}
\right. ,
\end{equation} 
\begin{equation}
\label{eq::equation_2}
\begin{pmatrix}
a_{22} & -a_{32} \\
a_{23} & -a_{33} \\
\end{pmatrix}
\begin{pmatrix}
a'_{31} \\
a'_{21}
\end{pmatrix}
 = \mathbf{0}
\end{equation}
and 
\begin{equation}
\label{eq::equation_3}
\begin{pmatrix}
a'_{22} & -a'_{32} \\
a'_{23} & -a'_{33} \\
\end{pmatrix}
\begin{pmatrix}
a_{31} \\
a_{21}
\end{pmatrix} = \mathbf{0}.
\end{equation}
Based on Eq.(\ref{eq::equation_1}), we can obtain the conclusion that
\begin{equation}
\begin{pmatrix}
a_{22} &  a_{23} \\
a_{32} &  a_{33} \\
\end{pmatrix}
 = \lambda 
 \begin{pmatrix}
a'_{22} &  a'_{23} \\
a'_{32} &  a'_{33} \\
\end{pmatrix}, ~~ \text{where}
\begin{gathered}
\lambda \neq 0.
\end{gathered}      
\end{equation}
According to Eq.(\ref{eq::equation_2}), we have
\begin{equation*}
\begin{pmatrix}
a_{22} & -a_{32} \\
a_{23} & -a_{33} \\
\end{pmatrix}
\begin{pmatrix}
a'_{31} \\
a'_{21}
\end{pmatrix}
= \mathbf{0},~\text{however,}~
\begin{gathered}
\textrm{det}
\begin{pmatrix}
a_{22} ~ -a_{23} \\
a_{32} ~ -a_{33} \\
\end{pmatrix}
\neq 0,
\end{gathered}      
\end{equation*}
therefore, we have $[a'_{21}~a'_{31}]^T = \mathbf{0}$. Similarly, we get $[a_{21}~a_{31}]^T = \mathbf{0}$. Finally, we obtain the conclusion described in Theorem \ref{theo:theorem_1}.
\end{proof}

Based on the above theorem, Eq.\eqref{Eq:epipolar_constriant_2R} can be expressed more generally as 
\begin{equation}
\label{Eq:realform_epipolar_constraint}
(\mathbf{A}'(\mathbf{x}_2) \mathbf{R}_2^T \mathbf{\Phi}_2(\mathbf{x}_{2}))^T [\mathbf{e}]_{\times} (\mathbf{A}(\mathbf{x}_1) \mathbf{R}_1^T \mathbf{\Phi}_1(\mathbf{x}_{1})) = 0.
\end{equation}
For each correspondence $(\mathbf{x}_1, \mathbf{x}_2)$, a unique transformation pair $(\mathbf{A}(\mathbf{x}_1), \mathbf{A}'(\mathbf{x}_2))$ can be chosen independently as long as they satisfy the constraint described in Theorem \ref{theo:theorem_1}. 

If we define $\mathbf{H}^{\mathbf{x}_{1}}_{1} = \mathbf{A}(\mathbf{x}_1)\mathbf{R}^{T}_{1}\mathbf{\Phi}_1(\mathbf{x}_{1}) ~~ \text{and} ~~ \mathbf{H}^{\mathbf{x}_{2}}_{2} = \mathbf{A}'(\mathbf{x}_2)\mathbf{R}^{T}_{2}\mathbf{\Phi}_2(\mathbf{x}_{2})$ as the homographies at $\mathbf{x}_1$ and $\mathbf{x}_2$ respectively, then Eq.\eqref{Eq:realform_epipolar_constraint} can be rewritten simply as
\begin{equation}
(\mathbf{H}^{\mathbf{x}_{2}}_{2})^{T}[\mathbf{e}]_\times \mathbf{H}^{\mathbf{x}_{1}}_{1} = 0.
\label{Eq:Spatially_variant_homography}
\end{equation}
Eq.\eqref{Eq:Spatially_variant_homography} can be viewed as a generalized form of Eq.\eqref{Eq:epipolar_constriant_bearingvector}. The only difference is that the transformation homography used here varies with the image location.
\subsection{Stereo rectification formulation \label{subsec:Formulation_SR}}
In this subsection, we proposed to apply the pixel-wise local homography to reduce the resampling distortion during the stereo rectification procedure. Without loss of generality, we denote ${\cal H}_{1,2}$ as a group of transformation matrices (e.g., $\mathbf{H}^{\mathbf{x}_{1,2}}_{1,2}$) to map the original image $\mathbf{I}_{1,2}$ to rectified one $\mathbf{I}'_{1,2}$ and the stereo rectification problem aims at finding these two groups of transformation matrices ${\cal H}_1$ and ${\cal H}_2$ which satisfy the following requirements:
\begin{enumerate}[1)]
\item \textbf{Epipolar line alignment:} After rectification, the epipolar lines should be aligned with the image scan lines in the rectified images. Mathematically, the epipolar line constraint can be satisfied if
\begin{equation}
({\cal H}_{2}(\mathbf{x}_{2}))^{T}[\mathbf{e}]_\times {\cal H}_{1}(\mathbf{x}_{1}) = 0,
\label{Eq:Spatially_variant_H_general}
\end{equation}
where $\mathbf{x}_{1}$ and $\mathbf{x}_{2}$ are correspondences in $\mathbf{I}_1$ and $\mathbf{I}_2$. ${\cal H}_{1}(\mathbf{x}_{1}) = {\mathbf{H}}^{\mathbf{x}_{1}}_{1}$ and ${\cal H}_{2}(\mathbf{x}_{2}) ={\mathbf{H}}^{\mathbf{x}_{2}}_{2}$ represent pixel-wise  homography matrices at $\mathbf{x}_{1}$ and $\mathbf{x}_{2}$ respectively. 

\item \textbf{Scan line order preserving:} Unlike the global homography based rectification, which can both satisfy the epipolar constraint and preserve the scene structure. The pixel-wise local homography can only guarantee the epipolar constraint due to its extra freedom. Therefore, scan line order preserving constraint is proposed here to preserve the scene structure before and after rectification. To some extend, this requirement can be satisfied if the whole mapping process is unique and invertible.  

\item \textbf{Resampling distortion minimization:} During the rectification, the image local information should also be preserved. Assuming the resampling distortion caused by ${\cal{H}}_{1}$ and ${\cal{H}}_{2}$ is ${\cal L}({\cal{H}}_{1},{\cal{H}}_{2})$, then a good transformation should minimize ${\cal L}_{}$ over the whole image.    
\end{enumerate}

By putting all the constraints together, we formulate the stereo rectification as the following minimization
\begin{equation}
\begin{aligned}
 & \underset{\mathcal{H}_{1}, \mathcal{H}_{2}}{\text{minimize}}
 & &  {\cal L}(\mathcal{H}_{1}, \mathcal{H}_{2}) \\
 & \text{subject to} 
 & &  ({\cal H}_{2}(\mathbf{x}_{2}))^{T}[\mathbf{e}]_\times{\cal H}_{1}(\mathbf{x}_{1}) = 0,\\ 
 & 
 & &  \mathcal{H}_1 ~\text{and}~ \mathcal{H}_2 ~\text{are invertable}.\\ 
 \end{aligned}
 \label{Eq:Minimization_1}
 \end{equation}
 
 \vspace{0.25cm}
\section{Implementation\label{sec:SR_solving}}
According to Eq.\eqref{Eq:realform_epipolar_constraint}, we find that $\cal H$ can be expressed as
\begin{equation} 
\cal {H} = \mathbf{\Psi} \circ \mathcal{R} \circ \mathbf{\Phi},
\end{equation} 
where $\mathbf{\Psi}$ represents the local transformation (e.g., $\mathbf{A}$, $\mathbf{A}'$) and $\mathcal{R}$ denotes the global rotation (e.g., $\mathbf{R}_1$, $\mathbf{R}_2$). The camera intrinsic parameters $\mathbf{\Phi}$ can be ignored here because it is constant for all image pixels.  

\subsection{Epipolar line alignment}
Generally speaking, global transformation $\mathcal{R}$ aims at rotating the left and right camera coordinates to force all the bearing vectors satisfying the epipolar constraint with a special form of $\mathbf{F} = [\mathbf{e}]_{\times}$ in the new camera coordinates. The rigid rotation $\mathcal{R}$ can be realized as two rotations $\mathbf{R}_1$ and $\mathbf{R}_2$, which can be computed from the camera relative pose \cite{bouguet2004camera}. The relative camera pose between two cameras can be estimated via off-line stereo calibration or on-line estimation from the two images. After rotating with $\mathbf{R}_1$ and $\mathbf{R}_2$, all the bearing vectors satisfy the epipolar line constraint below
\begin{equation}
(\mathbf{R}^{T}_2 \mathbf{\Phi}_{2}(\mathbf{x}_{2}))^{T}[\mathbf{e}]_\times(\mathbf{R}^{T}_1\mathbf{\Phi}_{1}(\mathbf{x}_{1})) = 0.
\end{equation}
In addition, there exists an extra degree of freedom for rotation pair $(\mathbf{R}_1, \mathbf{R}_2)$, which is that any rotation around the baseline will not change the epipolar constraint.
\subsection{Scan line order preserving}
Local transformation $\mathbf{\Psi}$ (e.g., $\mathbf{A}$ and $\mathbf{A}'$), which can be considered as two projection matrices, aims at projecting bearing vectors into the new camera coordinates as $(\widehat{u}, \widehat{v}, 1)^T = \mathbf{A}(b_{x},b_{y},b_{z})^T$. Furthermore, we find that $\widehat{u}$ and $\widehat{v}$ can be determined independently in Theorem \ref{theo:theorem_2}. 
\begin{theorem}
\label{theo:theorem_2}
For any bearing vector $\mathbf{b} = (b_{x},b_{y},b_{z})$, its corresponding coordinates $(\widehat{u}, \widehat{v})$ in rectified image can be determined independently. 
\end{theorem}
\begin{proof}
Assuming any pair of matrices $\mathbf{A}$ and $\mathbf{A}'$ as defined in Theorem \ref{theo:theorem_1}, the coordinates $\widehat{\mathbf{x}}= (\widehat{u},\widehat{v},1)$ in rectified image can be computed from $\widehat{\mathbf{x}} = \mathbf{A} \mathbf{b}$ as 
 \begin{equation} 
 \label{Eq:projection_bearingvectors}
 \left\{
 \begin{aligned}
 \widehat{u} & =  \frac{a_{11}b_{x}+a_{12}b_{y}+a_{13}b_{z}}{a_{32}b_{y}+a_{33}b_{z}} \\
 \widehat{v} & = \frac{a_{22}b_{y}+a_{23}b_{z}}{a_{32}b_{y}+a_{33}b_{z}}
 \end{aligned}
 \right. .
 \end{equation} 
 \normalsize
Although $\widehat{u}$ and $\widehat{v}$ share some variables in Eq.(\eqref{Eq:projection_bearingvectors}), they can also be considered as independent because $a_{1j}, j=1,2,3$ are arbitrary. 
\end{proof}

Without loss of generality, we use $\mathbf{\Psi}_{u}$ and $\mathbf{\Psi}_{v}$ to represent the projection of bearing vector into $u$ and $v$ coordinates respectively. Mathematically, we have
$$\widehat{u} = \mathbf{\Psi}_{u}(\mathbf{b})~~\text{and} ~~\widehat{v} = \mathbf{\Psi}_{v}(\mathbf{b}).$$ 

Then the {\em local transformation} $\mathbf{\Psi}$ is comprised of two parts, namely $\mathbf{\Psi}_{u}$ and $\mathbf{\Psi}_{v}$. Frankly speaking, the choice of $\mathbf{\Psi}$ is infinite many. Here, we just introduce a reasonable and easy implemented solution for $\mathbf{\Psi}$ and leave the task of seeking a more general expression of $\mathbf{\Psi}$ for future work. 

After rotating with $\mathbf{R}_1$ and $\mathbf{R}_2$, two cameras reach a desired canonical configuration as in fig.(\ref{fig:ProjectionToImage}): all the axis are coincident with each other except a translation (baseline) in $X$-axis. We propose to implement $\mathbf{\Psi}_{u}$ and $\mathbf{\Psi}_{v}$ with two angles ($\gamma, \beta$) which are defined as $\beta = \arctan\frac{b_{y}}{b_{z}}$ and $\gamma = \arctan\frac{b_{x}}{\sqrt{(b_{y})^2+(b_{z})^2}}$, where $(b_{x},b_{y},b_{z})$ is any bearing vector in the new camera coordinates. We implement $\mathbf{\Psi}_{v}$ with $\beta$ can keep the epipolar line constraint unchanged after projection  because all points on the same epipolar line share same $\beta$. In addition, implementing $\mathbf{\Psi}_{u}$ with $\gamma$ can keep the order of the points on the same epipolar line. 
\begin{figure}[!ht]
\centering
\includegraphics[width=0.25\textwidth]{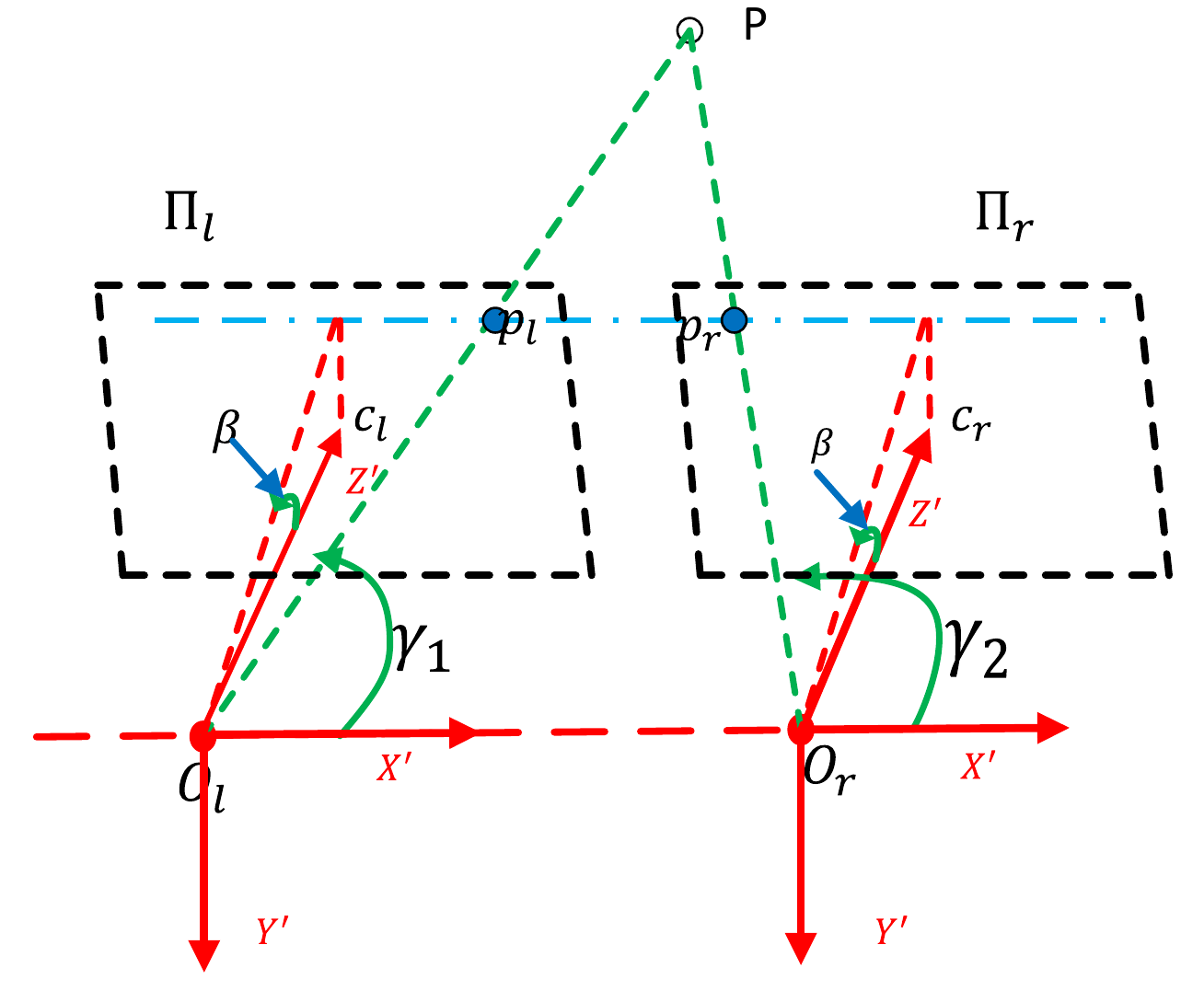} 
\caption{Re-project bearing vectors into new camera coordinates.}
\label{fig:ProjectionToImage}
\end{figure}

In \cite{kannala2006generic,scaramuzza2006toolbox}, general camera models with polynomial forms have been proposed to describe the different camera lenses. Inspired by their idea, we propose to realize both $\mathbf{\Psi}_{u}$ and $\mathbf{\Psi}_{v}$ with cubic functions as
\begin{equation}
 \label{Eq:Poly_nomial_rectification_Models}
\mathbf{\Psi}_{u,v} =  ~ c_{0} + c_{1}f(\theta)+ c_{2}f^2(\theta) +  c_{3}f^3(\theta),  \theta \in \{\beta \ \text{or} \ \gamma\},
\end{equation}
$\mathbf{c} = (c_0,~c_1,~c_2,~c_3)$ are coefficients required for estimation and $f(.)$ is a certain function of $\theta$, which can be of any form as long as it is monotonous within the whole field of view $[-0.5\theta_{max},0.5\theta_{max}]$. In this paper, we choose $f(\theta) = \theta$ for all experiment, which is similar with the Equidistance projection model \cite{kannala2006generic}. Particularly, in order to maintain uniqueness of the projection and preserve the scan lines order, both $\mathbf{\Psi}_{u}(.)$ and $\mathbf{\Psi}_{v}(.)$ must be monotonous within $[-0.5\theta_{max},0.5\theta_{max}]$.

\emph{Note:} Revisiting Eq.(\ref{Eq:projection_bearingvectors}), we find that $\widehat{u}$ and $\widehat{v}$ are functions of $b_{x}$, $b_{y}$, $b_{z}$ and coefficient matrix $\mathbf{A}$. The proposed cubic models $\mathbf{\Psi}_{u,v}$ can be considered as an approximation of Eq.(\ref{Eq:projection_bearingvectors}) because both $\mathbf{\Psi}_{u}$ and $\mathbf{\Psi}_{v}$ are expressed with $b_{x}$, $b_{y}$, $b_{z}$. 
\subsection{Resampling distortion minimization\label{subsec:Resampling_em}}
 \begin{figure}[hb!]
 \centering
 \includegraphics[width=0.3\textwidth]{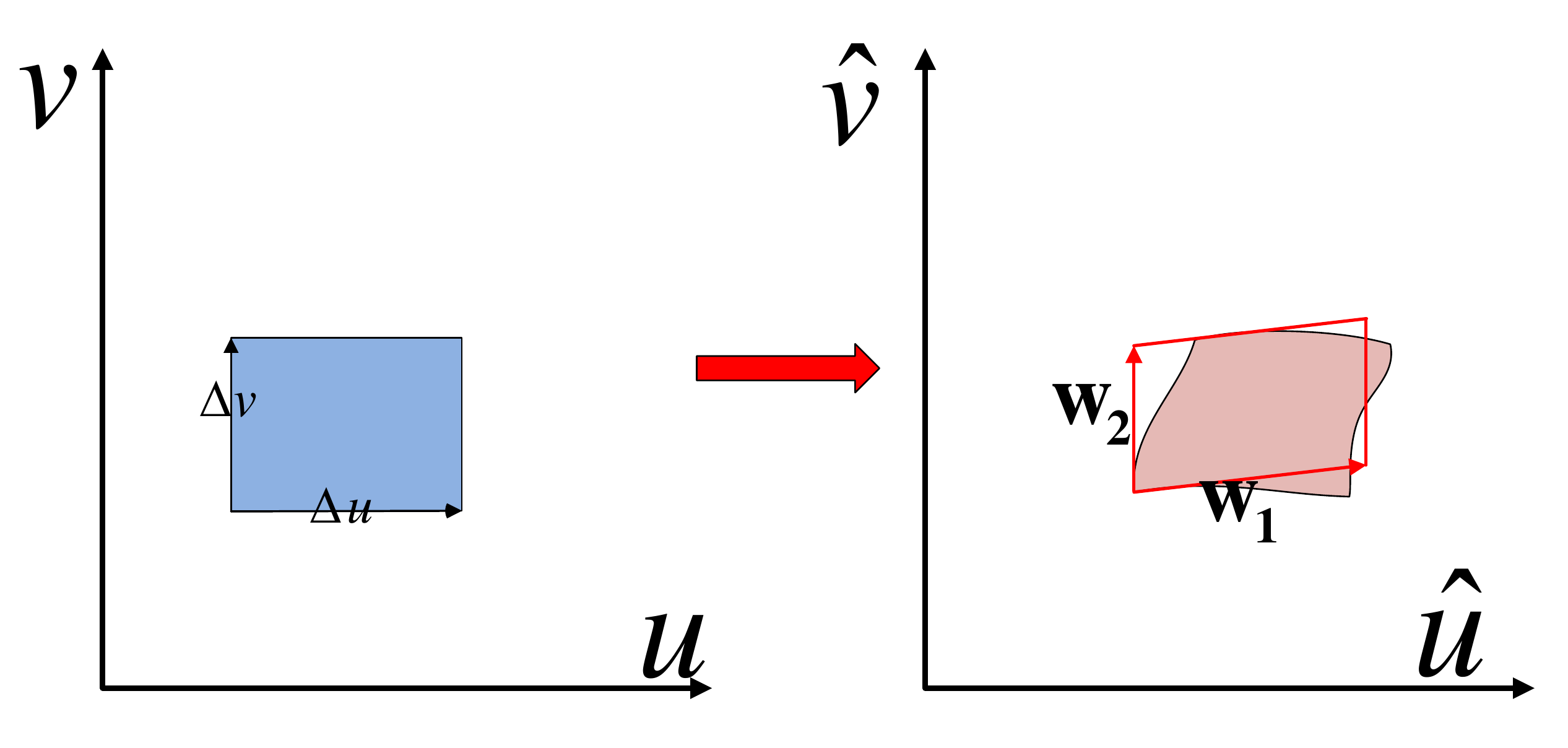} 
 \caption{Transforming a unit rectangle from original image into an approximate parallelogram in rectified image.}
 \label{fig:SmallPatchTransformation}
 \end{figure}
In \cite{gluckman2001rectifying}, Jacbian matrix has been employed to measure resampling distortion caused by  rectification which realized by using a global homography matrix. To facilitate understanding, the deformation of a unit rectangle is considered here. After rectification, a unit rectangle $(u,v)$ transformed via $\cal{H}$ to $(\widehat{u},\widehat{v})$ from original image to rectified image. As displayed in fig.(\ref{fig:SmallPatchTransformation}), the small region in the rectified image can be approximated by a parallelogram \cite{andrilli2010elementary} whose two sides are $\mathbf{w}_1 = [\frac{\partial \widehat{u}}{\partial u}, \frac{\partial \widehat{v}}{\partial u}]^T$ and $\mathbf{w}_2 = [\frac{\partial \widehat{u}}{\partial v}, \frac{\partial \widehat{v}}{\partial v}]^T$, which are tangent to the boundary of the transformed region. The resampling distortion of this small region can be measured from three aspects:
 \begin{enumerate}[1)]
 \item The area $S$ of this region is expect to be unchanged after rectification. $S$ can be computed as $S = |\mathbf{w}_1 \times \mathbf{w}_2| = |\frac{\partial \widehat{u}}{\partial u} \frac{\partial \widehat{v}}{\partial v}  - \frac{\partial \widehat{u}}{\partial v} \frac{\partial \widehat{v}}{\partial u}|.$ The rectangle will shrink or expand if $S$ is smaller or larger than $1$. Therefore, the loss for the change of area is defined as $\mathcal{L}_{area} = (S - 1)^2.$
 \item Meanwhile, the aspect ratio of the rectangle is also expect to be unchanged. Changing of this ratio will affect the distribution of pixels in two directions. The loss for this change is defined as 
 $\mathcal{L}_{ratio} = (|\mathbf{w}_{1}|-|\mathbf{w}_{2}|)^2.$
 \item Finally, the rectangle is also expected to be unskewed after rectification. The loss of skew deformation is defined as  
 $\mathcal{L}_{skew} = (\mathbf{w}_{1}^T\mathbf{w}_{2} - 0)^2.$
 \end{enumerate}
 Finally, the resampling distortion of the entire image can be obtained as
\begin{equation}
 {\cal L} = \int_{0}^{w}\int_{0}^{h}(\mathcal{L}_{area}+\alpha_{1}\mathcal{L}_{ratio}+ \alpha_{2}\mathcal{L}_{skew})dudv, 
 \label{Eq:resampling_distortion}
\end{equation} 
in which $\alpha_1$ and $\alpha_2$ are weight factors to balance of different kind of losses and $w$, $h$ are the width and height of image.

\subsection{Optimization \label{sec:Optimization}}
Finally, Eq.\eqref{Eq:Minimization_1} can be rewritten specifically as 
\begin{equation}
\small
\begin{aligned}
& \underset{\mathbf{R}_{1,2}, \mathbf{c}_{1,2}}{\text{minimize}}
& & \int_{0}^{w}\int_{0}^{h}(\mathcal{L}_{area}+\alpha_{1}\mathcal{L}_{ratio}+ \alpha_{2}\mathcal{L}_{skew})dudv \\ 
& \text{subject to}
& & (\mathbf{\Psi}_2(\mathbf{R}^{T}_{2} \mathbf{\Phi}_{2}(\mathbf{x}_{2i})))^{T}[\mathbf{e}]_\times\mathbf{\Psi}_1(\mathbf{R}^{T}_1\mathbf{\Phi}_{1}(\mathbf{x}_{1i})) = 0,\\
&
& & \mathbf{\Psi}_{1,2}^{'}(\theta) > 0, \  \theta \in [-0.5\theta_{max},0.5\theta_{max}],\\
\end{aligned}
\label{Eq:Minimization}
\end{equation}
\normalsize
where $\mathbf{c}_{1,2}$ are the coefficients of $\mathbf{\Psi}_{1,2}$ and $\mathbf{\Psi}_{1,2}^{'}(.)$ represents as the derived function of $\mathbf{\Psi}_{1,2}$.

Minimization of Eq.\eqref{Eq:Minimization} is non-linear optimization problem with two constraints: monotonic and epipolar line constraint. In the real application, the monotonic constraint of $\mathbf \Psi $ is technically realized by some control points which are uniformly selected in $[-0.5\theta_{max},0.5\theta_{max}]$. In addition, the epipolar constraint can be achieved via relative pose ($\mathbf{R}$,$\mathbf{t}$) \cite{li2006five,yang2014optimal} estimation and following local transformation $\mathbf{\Psi}_1$ and $\mathbf{\Psi}_2$ will not change the epipolar line constraint due to our special implementation. However, due to noise in features extraction, matching and relative pose estimation process, the epipolar line constraint is not exactly satisfied. Then, we relax the epipolar constraint as 
\begin{equation}
 |\mathbf{\Psi}_2(\mathbf{R}^{T}_{2} \mathbf{\Phi}_{2}(\mathbf{x}_{2i})))^{T}[\mathbf{e}]_\times\mathbf{\Psi}_1(\mathbf{R}^{T}_{1}\mathbf{\Phi}_{1}(\mathbf{x}_{1i})| \leq \epsilon.
 \end{equation}
A point is considered to be satisfied epipolar constraint as long as the displacement in $y$ direction below a certain threshold $\epsilon$ (e.g., 1 pixel).

After obtaining $\mathbf{R}_1$ and $\mathbf{R}_2$, only the coefficients $\mathbf{c}_{1,2}$ are required to be optimized. Three types of loss $L_{area}$, $L_{ratio}$ and $L_{skew}$ are defined in subsection \ref{subsec:Resampling_em}, which can be easily computed by taking $\frac{\partial  \mathcal{H}(\mathbf{x})}{\partial u}$ and $\frac{\partial \mathcal{H}(\mathbf{x})}{\partial v}$. Theoretically, all the image pixels should be taken into consideration. For efficiency purpose, around 500 pixels evenly distributed all over the image are selected for computing the loss. We empirically choose $\alpha_1 = \alpha_2 = 0.5$ for all our experiments. 

Finally, Eq.\eqref{Eq:Minimization} is solved by employing a classical constrained nonlinear optimization solver ``Interior Point Algorithm'' \cite{byrd1999interior}, which can converge after hundreds of iterations. Obviously, the loss defined in Eq.\eqref{Eq:Minimization} is non-convex because it contains quadratic and second cross terms. However, in the real implementation, we find that the optimization is not sensitive to the initial values, even taking randomly values as inputs. Two examples of energy curve are shown in fig.(\ref{fig:LossChanges}). From these figures, we can find that the energy function is nearly convex over a large range.
 \begin{figure}[t!]
 \centering
 \begin{subfigure}[t]{0.23\textwidth}
 \includegraphics[width=1\textwidth]{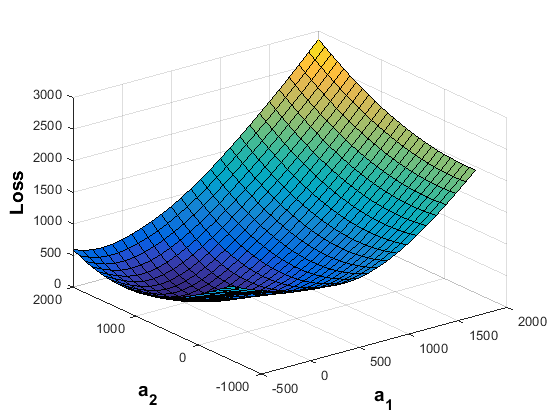} 
  \caption{}
 \label{subfig:Losschange_1}
 \end{subfigure} 
 ~
 \begin{subfigure}[t]{0.23\textwidth}
 \includegraphics[width=1\textwidth]{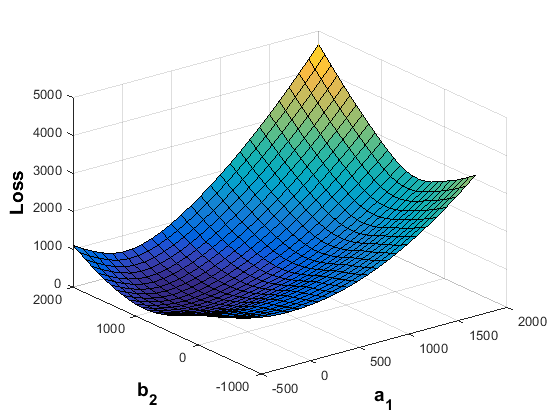} 
 \caption{}
 \label{subfig:Losschange_2}
 \end{subfigure}
\caption{Energy curves by changing two variables and keeping other variables constant.}
\label{fig:LossChanges}
\end{figure}

\vspace{0.2cm}
\section{Rectification and 3D Reconstruction\label{sec:3D_reconstruction}}
After obtaining the optimal $\mathbf{R}_{1,2}$ and $\mathbf{c}_{1,2}$, the rectification and 3D reconstruction can be implemented easily. 
\subsection{Stereo Rectification}
The expression of ${\cal H}$ can be easily obtained after obtain $\mathbf{R}_{1,2}$ and $\mathbf{c}_{1,2}$. At the same time, the backward ${\cal H}^{-1}$ of ${\cal H}$ is also required for rectification. ${\cal H}^{-1}$ is responsible for mapping a pixel $\mathbf{x}$ from the rectified image to the original one. This transformation indeed exists because ${\cal H}$ is monotonous within $[-0.5\theta_{max},0.5\theta_{max}]$. Furthermore, ${\cal H}^{-1}$ can also be expressed with a polynomial which can be obtained similarly as ${\cal H}$. Given ${\cal H}^{-1}$, the rectified image can be constructed pixel by pixel from its corresponding location in the original image. Finally, the main steps of our proposed rectification method can be summarized in Algorithm \ref{alg:Fisheye_Lens_Rectification}.
\begin{algorithm}[ht!]
\caption{Stereo Rectification \label{alg:Fisheye_Lens_Rectification}}
\begin{algorithmic}[1]
\Require{%
\begin{minipage}[t]{1\textwidth}
- A pair of image $(\mathbf{I}_{1}, \mathbf{I}_{2})$;\\
- Intrinsic camera projection model $\mathbf{\Phi}$;
\end{minipage}}
\Ensure{%
\begin{minipage}[t]{1\textwidth}
~~- Rectified image pair ($\mathbf{I}^{'}_{1}$,$\mathbf{I}^{'}_{2}$) and model $\mathcal{H}(.)$; 
\end{minipage}}
\noindent \begin{raggedright}
\rule[0.15dd]{0.999\linewidth}{1pt}
\par\end{raggedright}
\State $\blacktriangleright\;$Robust sparse features matching between $\mathbf{I}_{1}$ and $\mathbf{I}_{2}$;
\State $\blacktriangleright\;$Robust camera pose estimation \cite{li2006five};
\State $\blacktriangleright\;$Compute $\mathbf{R}_1$ and $\mathbf{R}_2$ based on \cite{bouguet2004camera} for rigid transformation;
\State $\blacktriangleright\;$Compute the non-linear transformation model $\mathbf{\Psi}$ by minimizing Eq.(\ref{Eq:Minimization}); 
\State $\blacktriangleright\;$Compute the backward model ${\cal H}^{-1}$ ;
\State $\blacktriangleright\;$Build the rectified image pair ($\mathbf{I}^{'}_{1}, \mathbf{I}^{'}_{2}$) via ${\cal H}^{-1}$.
\end{algorithmic}
\end{algorithm}
\normalsize
\subsection{3D Reconstruction}
After rectification, all the correspondences are aligned on the same scan lines. Traditionally dense matching approaches \cite{zhang2015meshstereo} can be employed directly for disparity computation. After obtaining the dense disparity map, 3D information can be reconstructed based on ${\cal H}^{-1}$. Assuming a world point $\mathbf{P}$, $(X_1, Y_1, Z_1)$ and $(X_2, Y_2, Z_2)$ are its corresponding coordinates in left and right rectified camera coordinates, then we have $[X_2, Y_2, Z_2]^T = [X_1 - b, Y_1, Z_1]^T$. Here, $b$ is the baseline between two cameras which can be obtained from calibration or relative pose estimation process. $(u_{1}, v_{1})$ and $(u_{2}, v_{2})$ are assumed to be the corresponding image locations in the left and right rectified images, then we have $u_{2} = u_{1}-d$ and $v_{1} = v_{2}$, where $d$ is the computed disparity value at location $(u_{1}, v_{1})$. 

Because $f(\theta) = \theta$ is taken as the base function of $\mathbf{\Psi}(.)$ here, we have  
\begin{equation*}
\gamma_1   = \mathbf{\Psi}^{-1}_{u}(u_{1}), ~ \gamma_2   = \mathbf{\Psi}^{-1}_{u}(u_{1} - d),~ \beta = \mathbf{\Psi}^{-1}_{v}(v_{1}).
\end{equation*}
Furthermore, as displayed in fig. \ref{fig:ProjectionToImage}, $\beta$ and $\gamma_{1,2}$ can also be expressed as
\begin{equation*} \beta = \arctan \frac{Y_1}{Z_1} = \arctan \frac{Y_2}{Z_2}, ~ \gamma_{1,2} = \arctan \frac{X_{1,2}}{\sqrt{Y^2_{1,2} + Z^2_{1,2}}}.  
\end{equation*}
Finally, 3D point $\mathbf{P}$ can be reconstructed as
\begin{equation}
\begin{aligned}
X_1 =   &  \frac{b \tan \gamma_1}{ \tan \gamma_1  - \tan \gamma_2}, ~ X_2 =  \frac{b \tan \gamma_2}{ \tan \gamma_1  - \tan \gamma_2},\\
Y   =   &  \frac{b\tan\beta}{(\tan \gamma_1  - \tan \gamma_2)\sqrt{1+\tan^2\beta}}, \\
Z   =   &  \frac{b}{(\tan \gamma_1  - \tan \gamma_2)\sqrt{1+\tan^2\beta}}. \\
\end{aligned}
\label{Eq:3D_reconstruction}
\end{equation}
\vspace{0.25cm}
\section{Experimental Results \label{sec:ExperimentalResults}}
The effectiveness  and robustness of our proposed method have been verified on different types of cameras: fisheye lens and traditional perspective images.  

\subsection{Rectification for real fisheye images}
Two types of real fisheye lens cameras are also used to evaluate our method. One is a Sony HDR-As200V motion camera (resolution: $1920 \times 1080$, lens: \SI{2.8}{\milli\meter}), offering an ultra-wide field of view with \ang{135} in horizontal direction and \ang{90} in the vertical direction. Another is a Canon camera (resolution $2736 \times 1842$, fisheye lens: \SI{8}{\milli\meter}) whose field of view is close to \ang{180}.

Different with a traditional stereo rig, two images from a motion (dominant motion is along X-direction) monocular camera are used here. The relative pose between two images are estimated by sparse matched features (e.g., SURF \cite{bay2006surf}). RANSAC and bucketing techniques are employed to remove outliers for robust estimation. Local bundle adjustment is also used to refine the camera pose and remove outliers and finally only the reliable correspondences are used for the following process. Two approaches are used for comparison here: 1) ``Conventional'': conventional method by using a perspective model (the focal length of the rectified image is chosen with minimal resampling distortion); 2) ``Abraham'': approach proposed by Abraham \emph{et al.} in \cite{abraham2005fish}, where the resampling distortion didn't take into consideration.  

Four pair of images have been taken for evaluation. The performances are evaluated on two aspects: rectification error and resampling distortion. The rectification error is measured as the average y-disparity of correspondences on the rectified image pair. For each rectified image pair, SURF features are used to obtain feature correspondences between the left and right images. During the matching process, only the feature correspondences whose matching distance is below a certain threshold (e.g., 1.5) are decided as reliable correspondences. The resampling distortion is specifically computed by using Eq.(\ref{Eq:resampling_distortion}). Similar with the optimization, around 500 pixels evenly distributed all over the image are selected for distortion computation. To be fair, same image pixels are used to compute resampling distortion for different methods. 

Quantitative evaluation results are shown in fig.(\ref{fig:Rec_eva_fish}). Subfig.(\ref{subfig:rec_error_fish}) illustrates the rectification error, where we can find that the rectification error is less than 1 pixel for all algorithms, which means that the epipolar line constraint is satisfied and all the correspondences are aligned on the same image row. Among the three, ``Abraham'' gives the smallest rectification error, while our proposed method is very close to their method. For the resampling distortion, our proposed method achieved the best performance all over the four examples. Furthermore, compared to the ``Conventional'' method, the distortion has been dramatically reduced. Compared with ``Abraham'' method, the resampling distortion has also been reduced about 30\% on the four examples. 

\begin{figure}[ht!]
 \centering
 \begin{subfigure}[t]{0.23\textwidth}
 \includegraphics[width=1\textwidth]{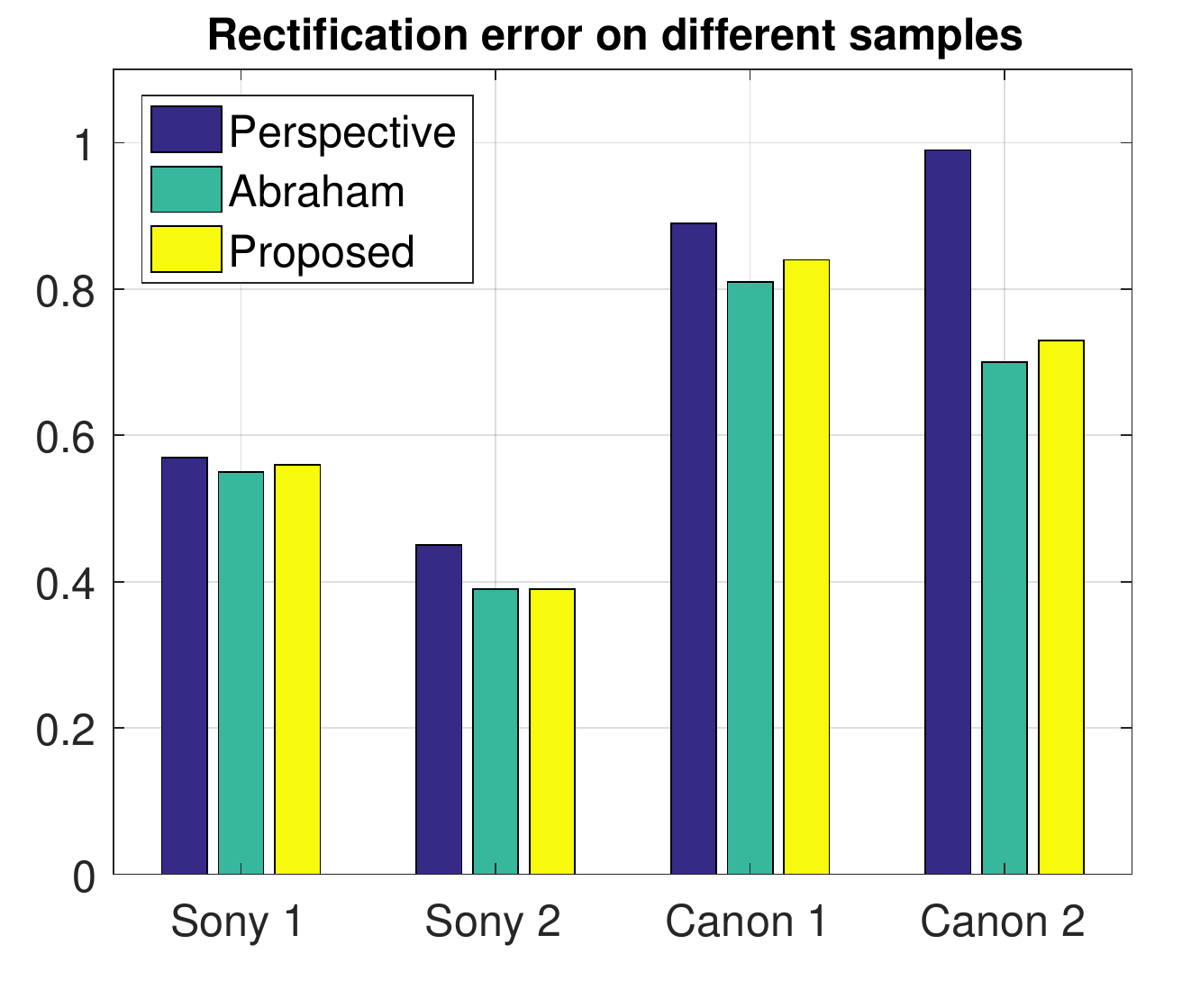} 
  \caption{}
 \label{subfig:rec_error_fish}
 \end{subfigure} 
 ~
 \begin{subfigure}[t]{0.23\textwidth}
 \includegraphics[width=1\textwidth]{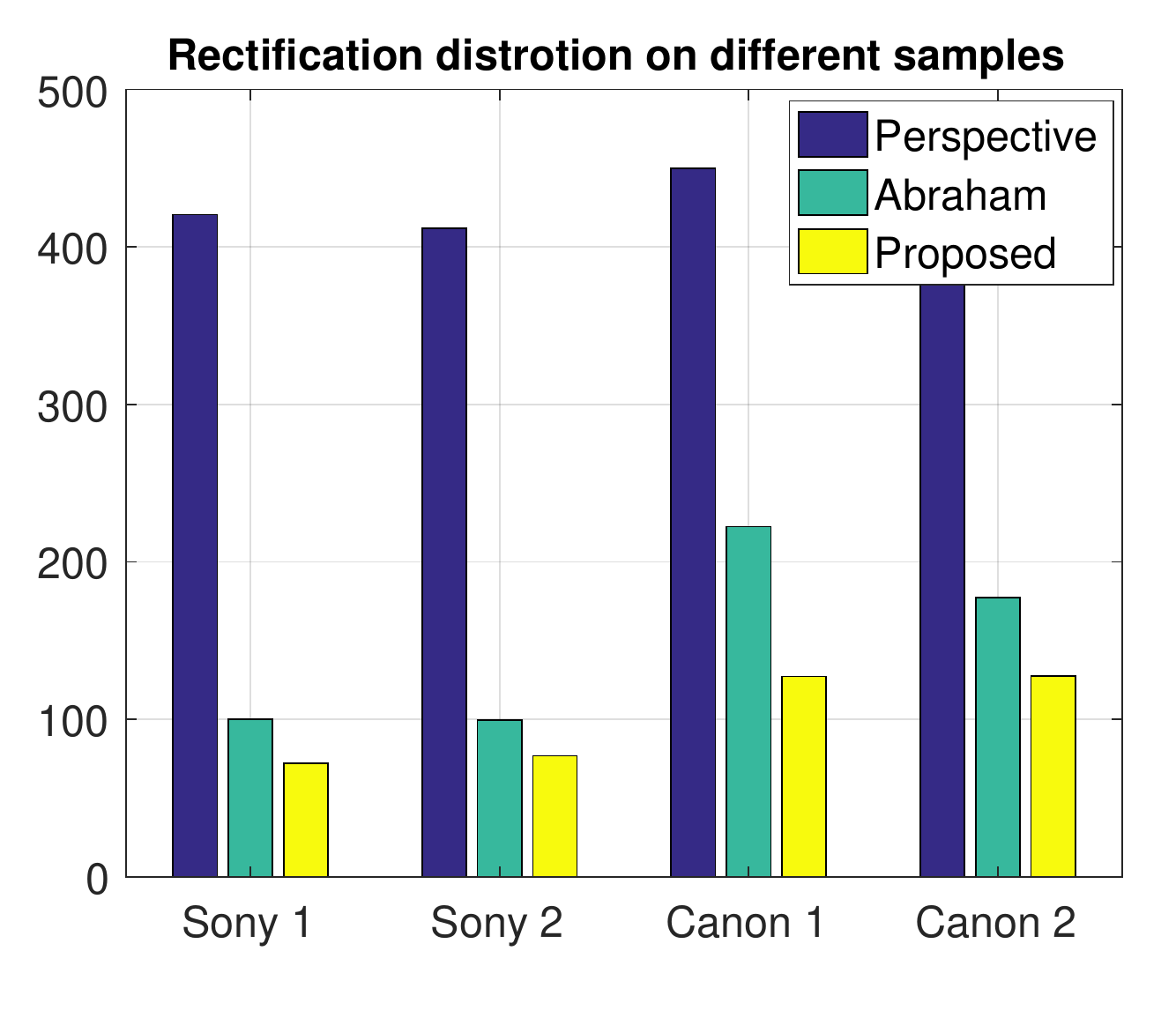} 
  \caption{}
 \label{subfig:rec_distor_fish}
 \end{subfigure}
 \caption{Rectification evaluation on real fisheye images. Subfig.(\ref{subfig:rec_error_fish}) and (\ref{subfig:rec_distor_fish}) give the rectification error and resampling distortion (Eq.(\ref{Eq:resampling_distortion})) for different approaches on real fisheye images.}
\label{fig:Rec_eva_fish}
\end{figure}
Fig.(\ref{fig:RectifiedImagesSonyCamera}) and fig.(\ref{fig:RectifiedImagesCanonCamera}) display two rectification results by using different methods. From these figures, we can obviously find that after rectification, the epipolar line has been well persevered. By using our proposed method, the local information has been largely persevered in the rectified image and the field of view of the rectified image is nearly the same with the original image. 
\begin{figure}[h!]
\centering
\begin{subfigure}[t]{0.475\textwidth}
\centering
\includegraphics[width=0.9\textwidth]{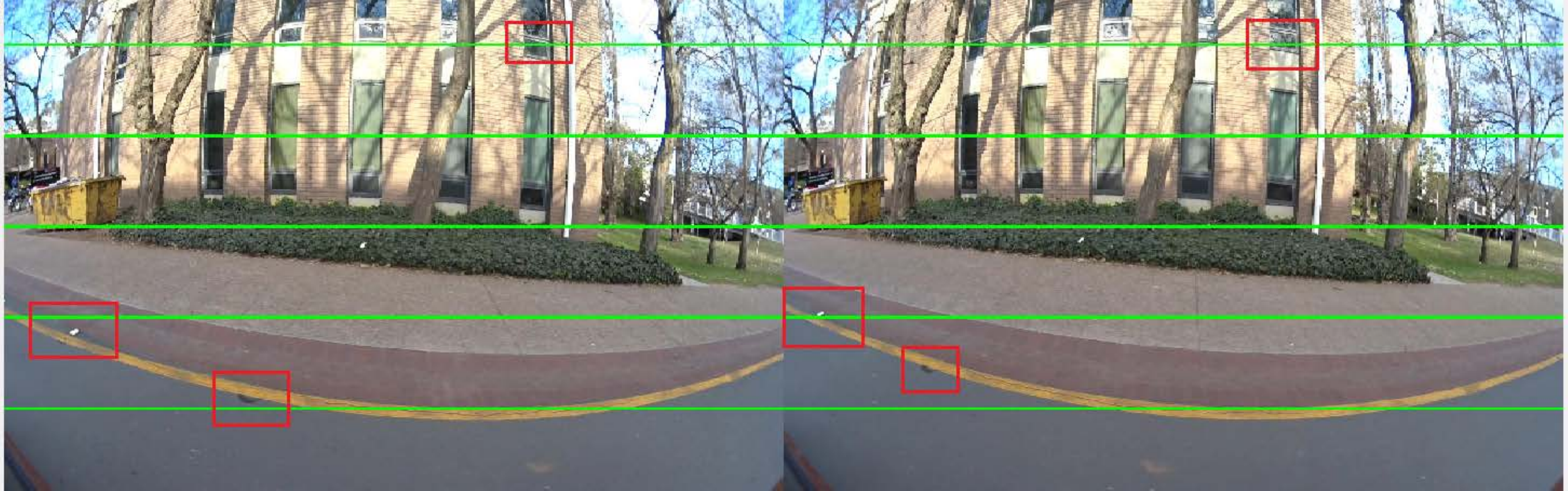} 
\caption{Fisheye image pair.}
\label{subfig:SonycamerOriginalImage}
\end{subfigure}
\begin{subfigure}[t]{0.475\textwidth}
\centering
\includegraphics[width=0.9\textwidth,height = 0.325\textwidth]{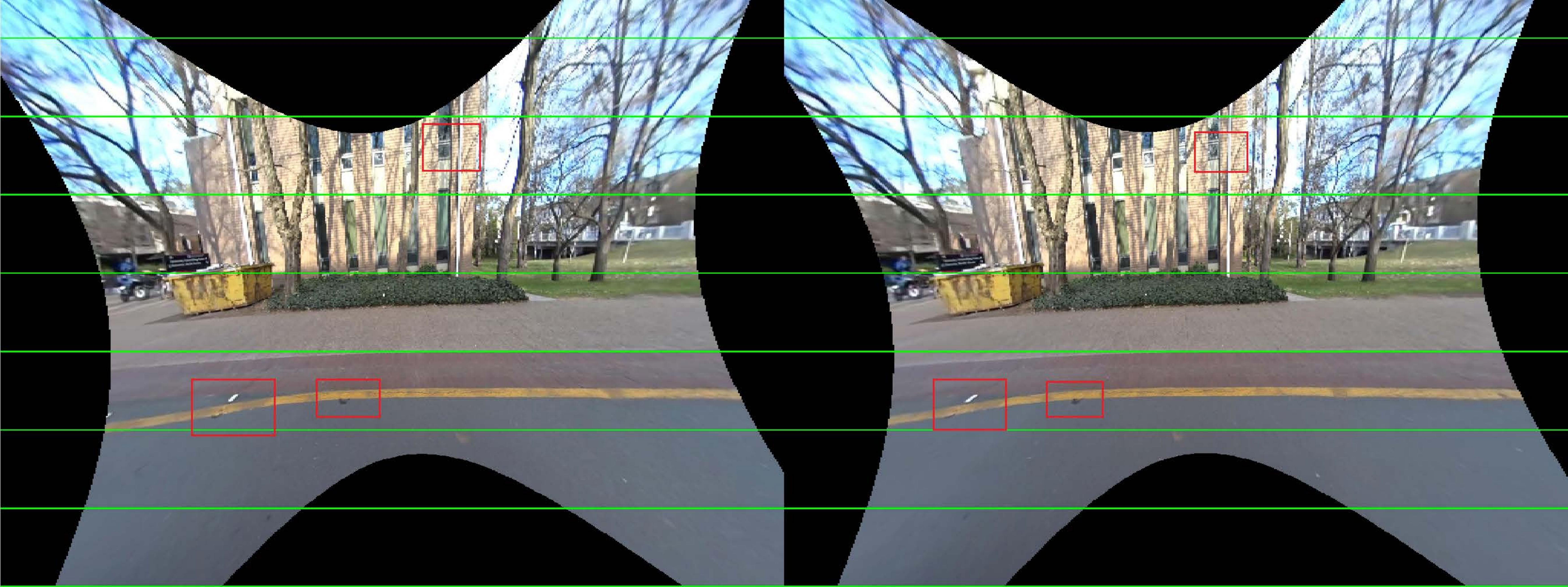} 
\caption{Rectified image pair using perspective model.}
\label{subfig:Perspective_rectification_snoy}
\end{subfigure}
\centering
\begin{subfigure}[t]{.475\textwidth}
\centering
\includegraphics[width=0.9\textwidth,height = 0.35\textwidth]{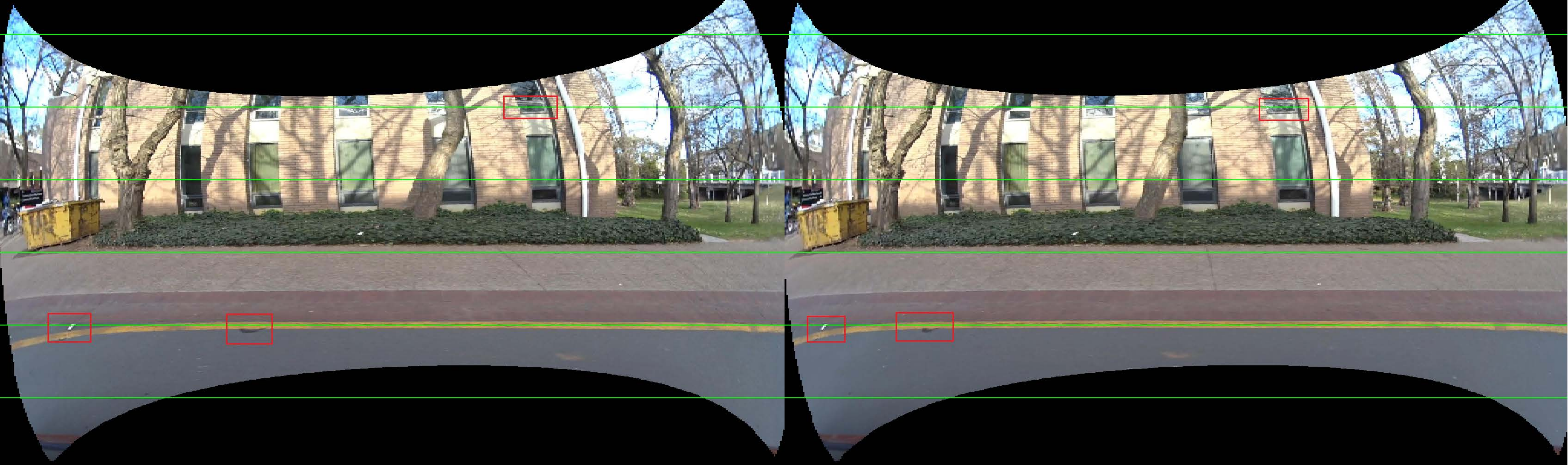} 
\caption{Rectified image pair using proposed method.}
\label{subfig:Proposed_rectification_snoy}
\end{subfigure}
\centering
\caption{Rectification results Sony Camera. Lines and rectangle are used to highlight that the epipolar constraint is satisfied after rectification.}
\label{fig:RectifiedImagesSonyCamera}
\end{figure}
\begin{figure}[h!]
\centering
\begin{subfigure}[t]{0.475\textwidth}
\centering
\includegraphics[width=0.9\textwidth]{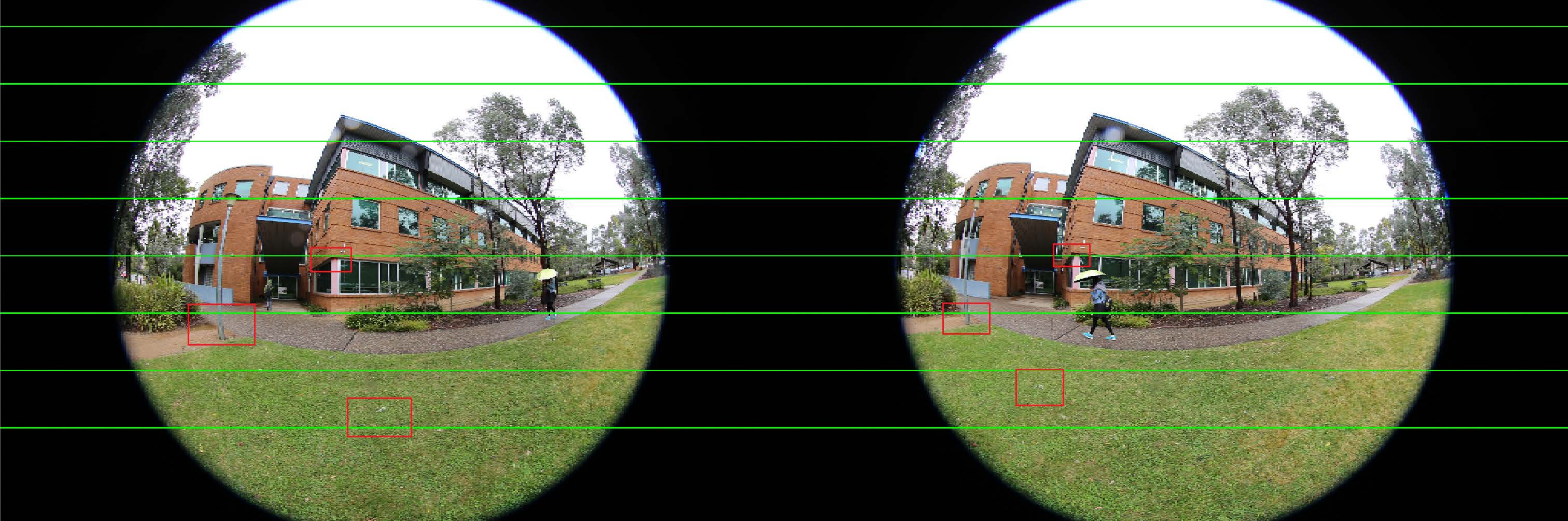} 
 \caption{Fisheye image pair.}
\label{subfig:CanoncamerOriginalImage}
\end{subfigure}
\begin{subfigure}[t]{0.475\textwidth}
\centering
\includegraphics[width=0.9\textwidth]{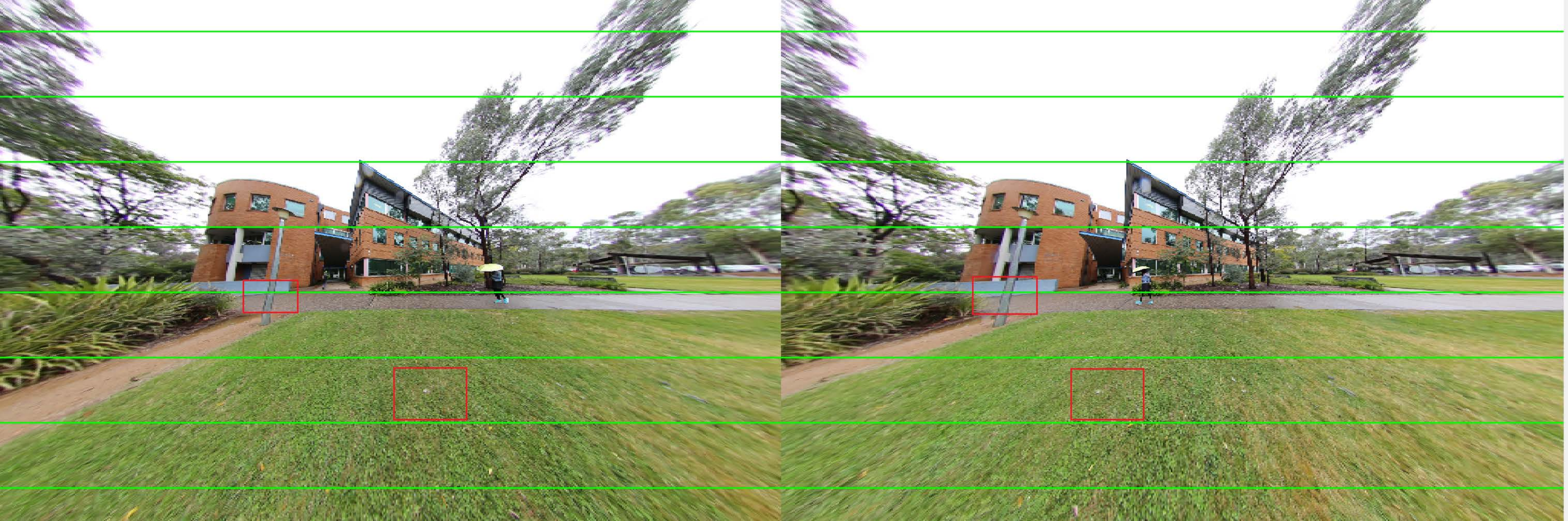} 
\caption{Rectified image pair using perspective model.}
\label{subfig:Perspective_rectification_Canon}
\end{subfigure}
\centering
\begin{subfigure}[t]{.475\textwidth}
\centering
\includegraphics[width=0.9\textwidth]{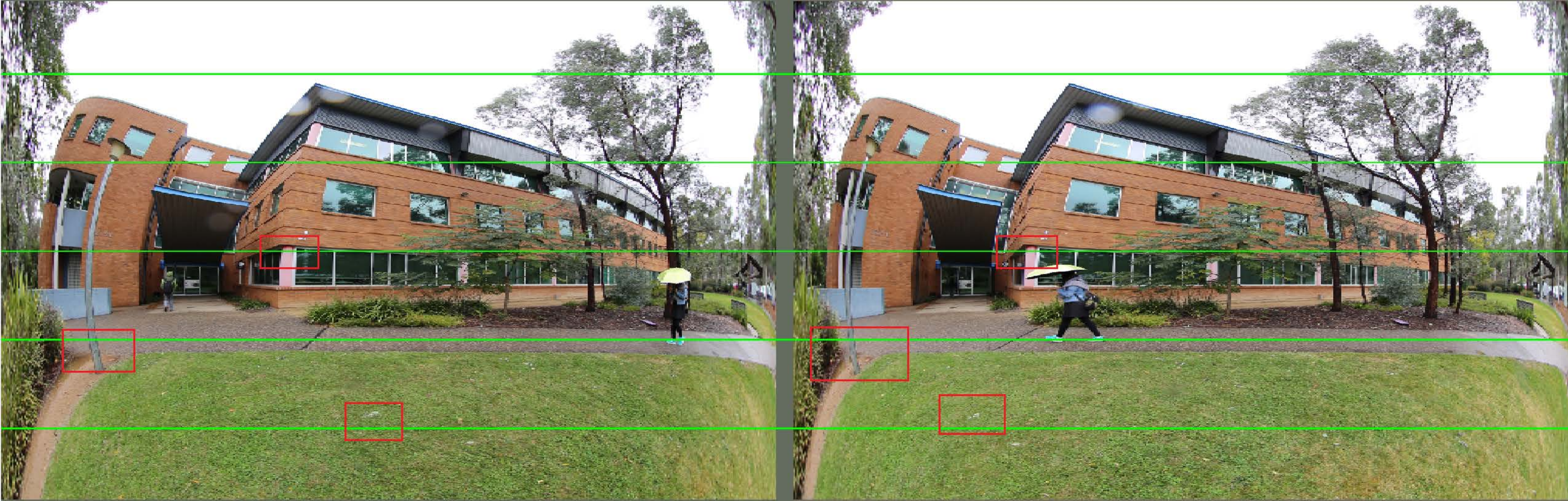} 
\caption{Rectified image pair using proposed method.}
\label{subfig:Proposed_rectification_Canon}
\end{subfigure}
\centering
\caption{Rectification results for Canon camera. Lines and rectangle are used to highlight that the epipolar constraint is satisfied after rectification.}
\label{fig:RectifiedImagesCanonCamera}
\end{figure}

\subsection{3D reconstruction result on real fisheye images}
The 3D reconstruction has been verified on real fisheye images. First, dense disparity map has been computed by using \cite{zhang2015meshstereo}, then Eq.\eqref{Eq:3D_reconstruction} is used for 3D reconstruction. Fig.\ref{fig:3D_rec} displays 3D reconstruction results with two different types of rectification methods. To save space, the original fisheye image pair and the rectified image pairs are corresponding to subfig.(\ref{subfig:SonycamerOriginalImage}), (\ref{subfig:Perspective_rectification_snoy}) and (\ref{subfig:Proposed_rectification_snoy}) respectively. Two different views of the reconstruction results are shown to demonstrate the performances of different methods.

Subfig.(\ref{subfig:Reconstruction}) and (\ref{subfig:Reconstruction_pers}) display the 3D reconstruction results by using proposed and traditional rectification methods respectively. From subfig.(\ref{subfig:Reconstruction}), we can see that the structure of the building, road, trees and vegetation have been reconstructed densely. The detailed structure of the trunk and the building's wall can be seen clearly by using Meshlab \cite{cignoni2008meshlab}. Qualitatively, as shown in subfig. \ref{subfig:Reconstruction}, the 3D reconstruction performs better at the center of the image due to less fisheye distortion and the accuracy of the reconstruction decrease with the increasing of the distortion. On the right side of the image, the color of the sky has been misaligned on the tree due to error in the dense disparity computation and the uncertainty in the rectification process, e.g., relative camera pose estimation. The performance can be improved by using a real stereo camera and the relative camera pose can be obtained via pre-calibration process.

\begin{figure}[ht!]
\centering
\begin{subfigure}[t]{0.475\textwidth}
\centering
\includegraphics[width=0.9\textwidth,height= 0.3\textwidth]{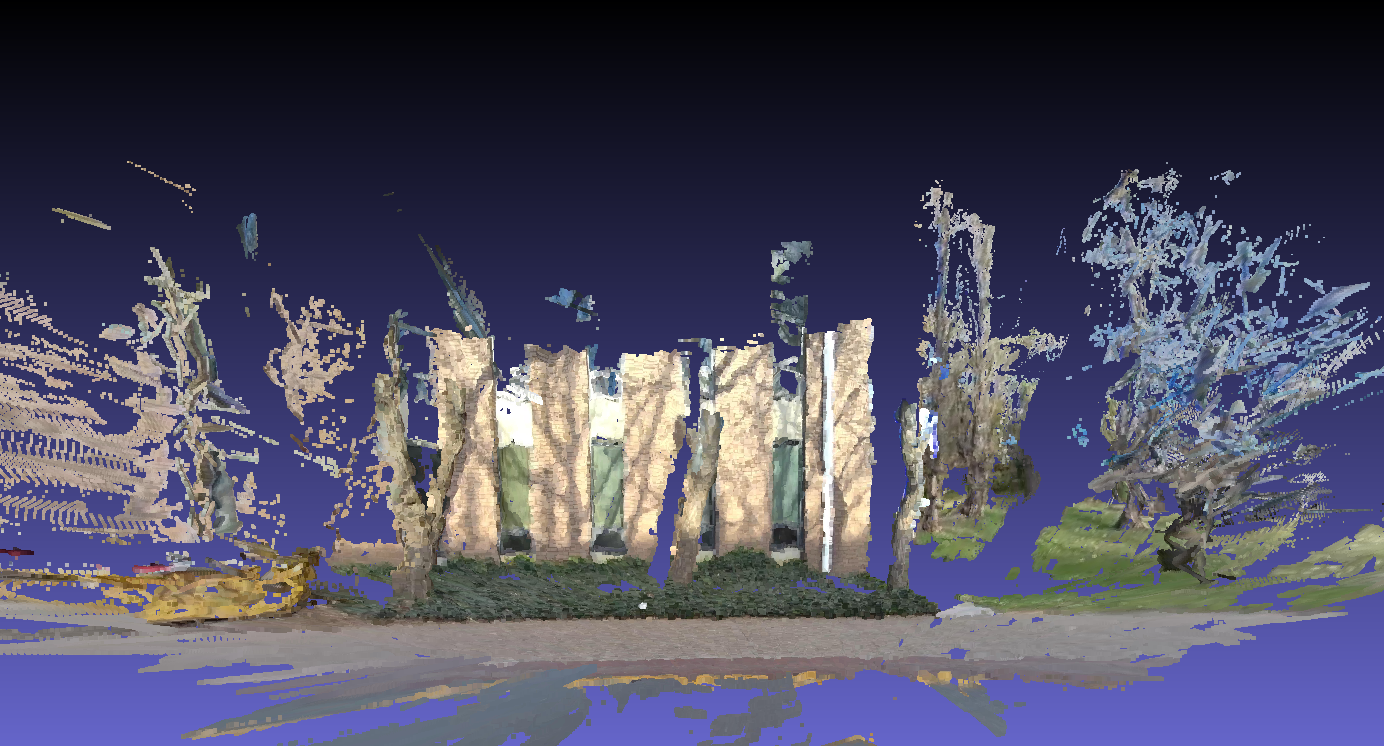}
\\
\centering
\includegraphics[width=0.9\textwidth,height= 0.3\textwidth]{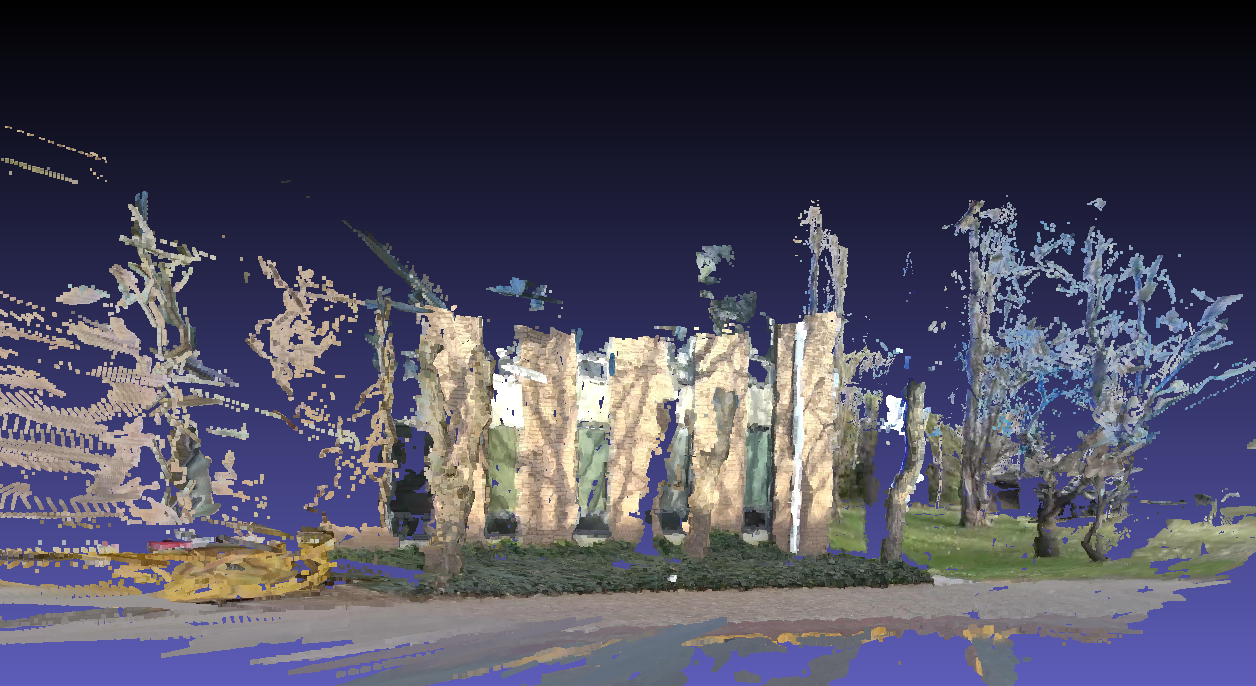} 
\caption{3D reconstruction by proposed rectification method.}
\label{subfig:Reconstruction}
\end{subfigure}
\begin{subfigure}[t]{0.475\textwidth}
\centering
		\includegraphics[width=0.9\textwidth,height= 0.3\textwidth]{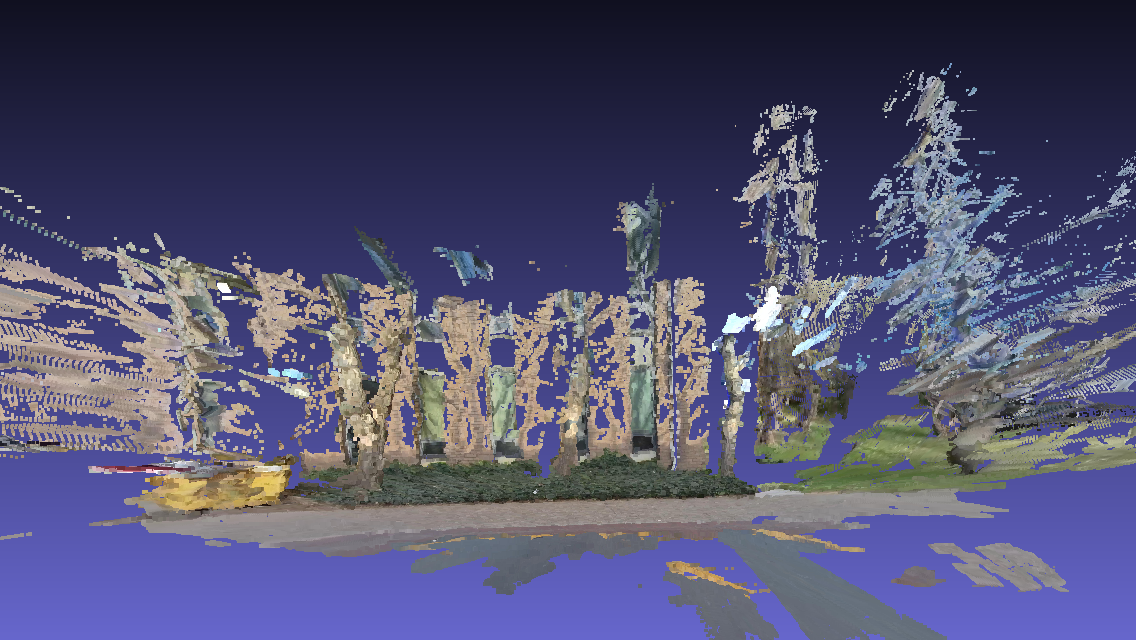}
		\\
        \centering
		\includegraphics[width=0.9\textwidth,height= 0.3\textwidth]{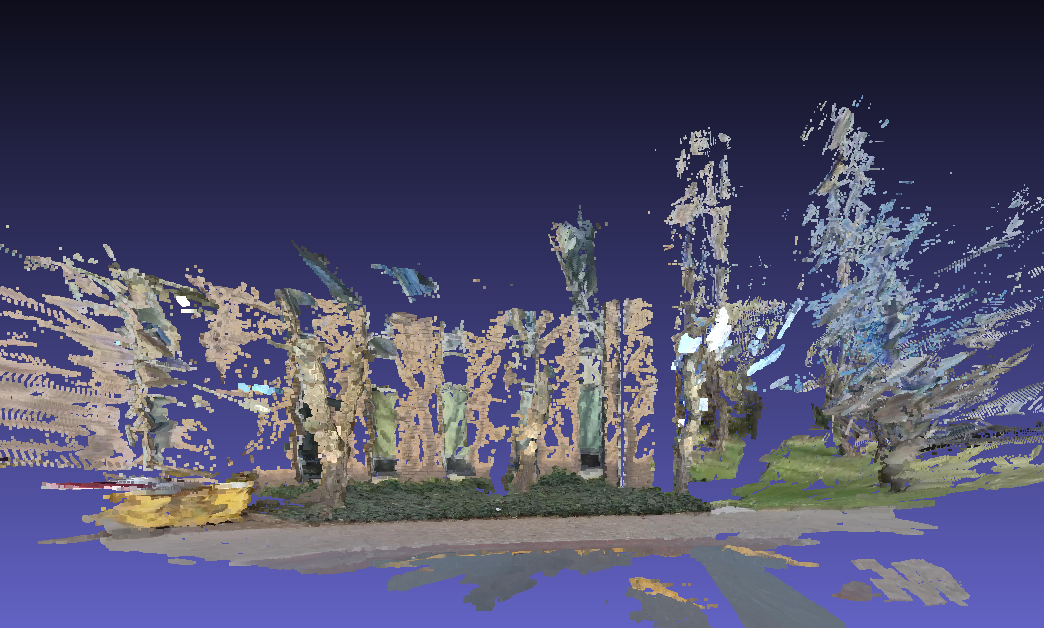} 
		\caption{3D reconstruction by traditional rectification method \cite{bouguet2004camera}.}
		\label{subfig:Reconstruction_pers}
	\end{subfigure}
\caption{3D reconstruction results on the real fisheye images by using different rectification methods. }
\label{fig:3D_rec}
\end{figure}
Comparing subfig.\ref{subfig:Reconstruction} with \ref{subfig:Reconstruction_pers}, we can see that the the former is qualitatively better than the latter, such as: 1) the building is reconstructed more dense in 3D because the texture information has been well kept during rectification in subfig.(\ref{subfig:Proposed_rectification_snoy}) while the texture details have been compressed in subfig.(\ref{subfig:Perspective_rectification_snoy}). 2) the trees on the right side of the image are reconstructed slightly better in subfig.(\ref{subfig:Reconstruction}) than subfig.(\ref{subfig:Reconstruction_pers}). That's because the resampling distortion is serious in subfig.(\ref{subfig:Perspective_rectification_snoy}) and this will hinder the following dense matching process. 
\subsection{Rectification results on perspective images}
Our proposed method can also be applied for perspective images. The public SYNTIM dataset \footnote{Webpage: http://perso.lcpc.fr/tarel.jean-philippe/syntim/paires.html} has been used to evaluate our proposed method. Two typical global homography based rectification methods: ``Hartley'' \cite{hartley1999theory} and ``Fusiello'' \cite{fusiello2008quasi} have been taken for evaluation here.
\begin{figure}[h!]
 \centering
 \begin{subfigure}[t]{0.23\textwidth}
 \includegraphics[width=1\textwidth]{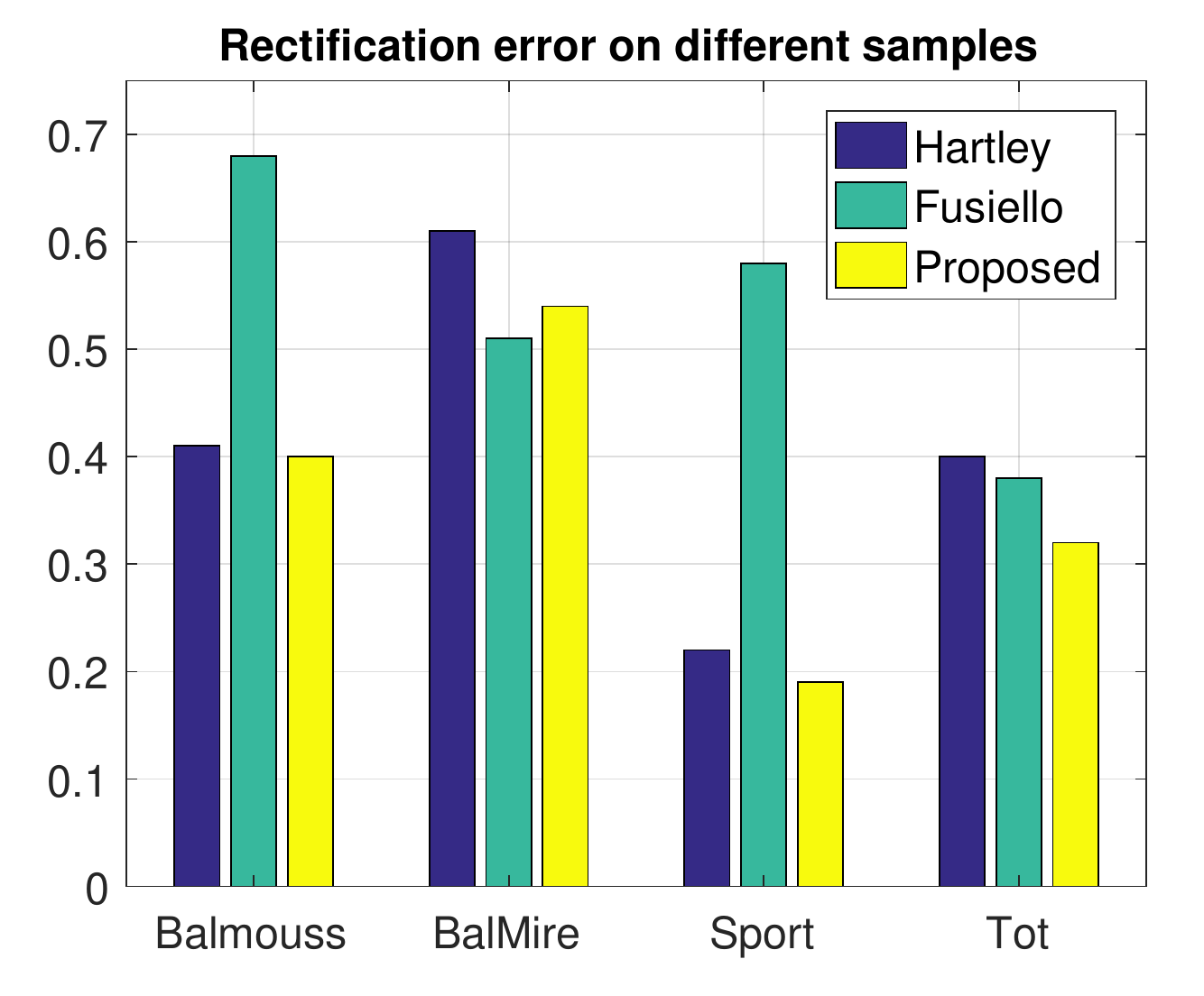} 
  \caption{}
 \label{subfig:rec_error_per}
 \end{subfigure} 
 ~
 \begin{subfigure}[t]{0.23\textwidth}
 \includegraphics[width=1\textwidth]{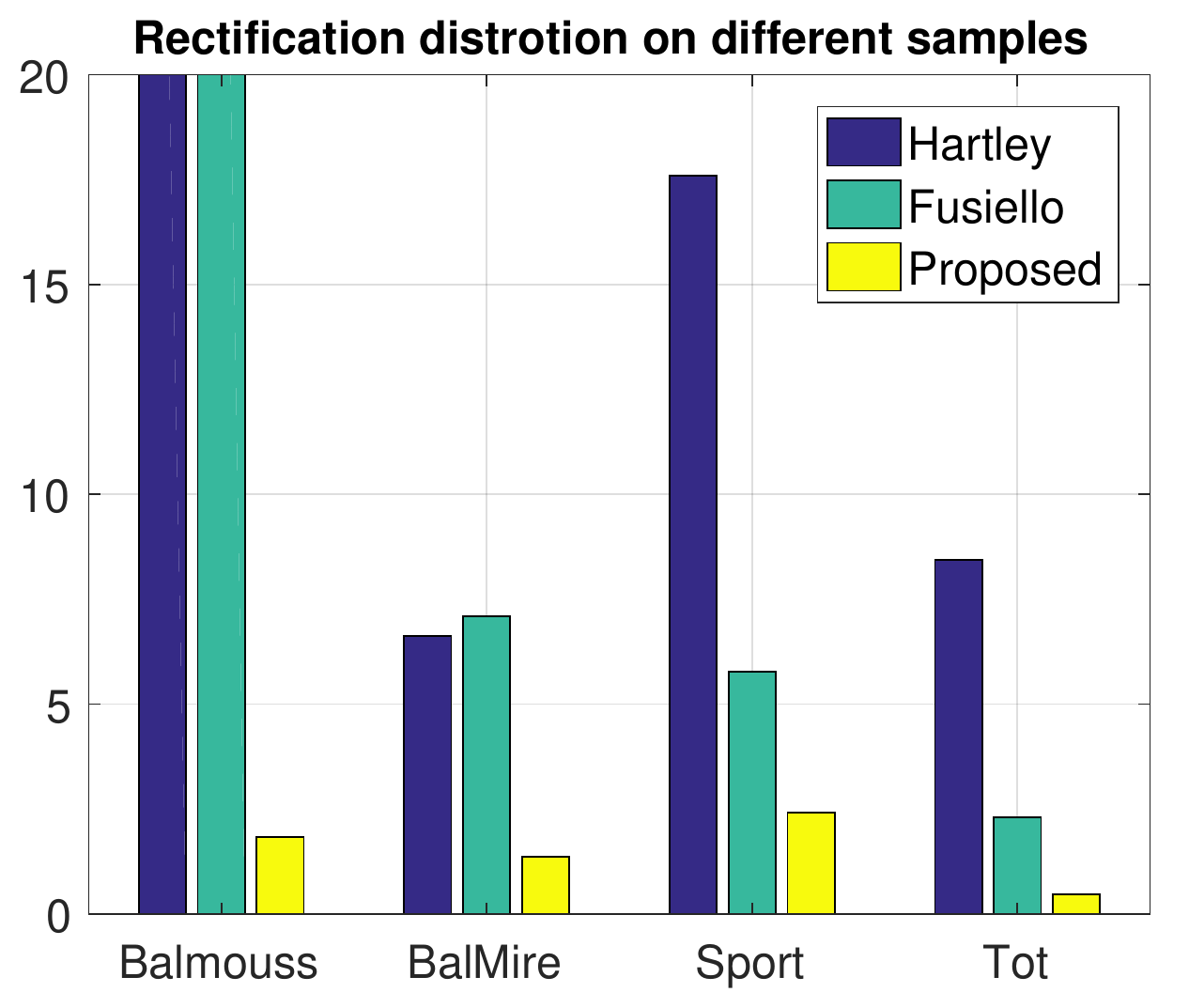} 
  \caption{}
 \label{subfig:rec_distor_per}
 \end{subfigure}
 \caption{Rectification evaluation on perspective images. Subfig.(\ref{subfig:rec_error_per}) and (\ref{subfig:rec_distor_per}) give the rectification error and resampling distortion (Eq.(\ref{Eq:resampling_distortion})) for different approaches on four images.}
 \label{fig:Rec_eva_per}
 \end{figure}
 
Fig.(\ref{fig:Rec_eva_per}) illustrates the evaluation results for four image pairs on SYNTIM dataset. Subfig.(\ref{subfig:rec_error_per}) and (\ref{subfig:rec_distor_per}) give the rectification error and resampling distortion for different methods. In subfig.(\ref{subfig:rec_error_per}), we find that the rectification error is small for all the methods (less than 1 pixel) and the proposed method gives best performance on three examples. Furthermore, as shown in subfig.(\ref{subfig:rec_distor_per}), the proposed method has dramatically reduced the resampling distortion on the four examples compared with the other two approaches. The rectification results of different methods on ``BalMouss'' have been shown in Fig.(\ref{fig:Rec_eva_per_exa}). Compared with other two methods, the proposed method has largely preserved local information of the original image by reducing the resampling distortion during rectification.  
\begin{figure}[h!]
\centering
\begin{subfigure}[t]{0.475\textwidth}
\includegraphics[width=0.475\textwidth]{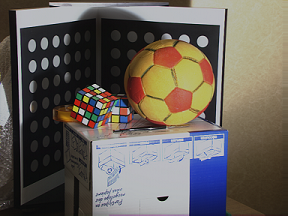}~
\includegraphics[width=0.475\textwidth]{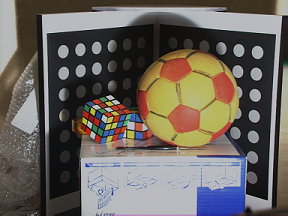} 
 \caption{Original image pair ``BalMouss'' in SYNTIM dataset.}
\label{subfig:BalMouss_per}
\end{subfigure}
\\
\begin{subfigure}[t]{0.475\textwidth}
\includegraphics[width=0.475\textwidth]{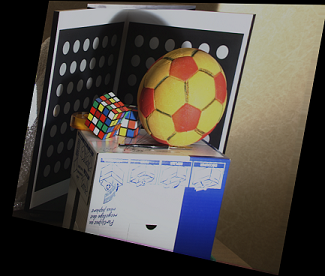} ~
\includegraphics[width=0.475\textwidth]{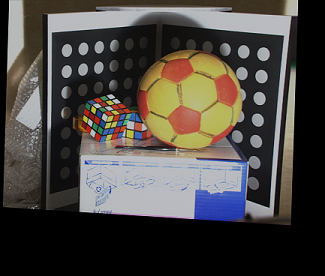} 
\caption{Rectification results with ``Hartley'' \cite{hartley1999theory} method.}
\label{subfig:rec_per_hartly}
\end{subfigure}
\\
\begin{subfigure}[t]{0.475\textwidth}
\includegraphics[width=0.475\textwidth]{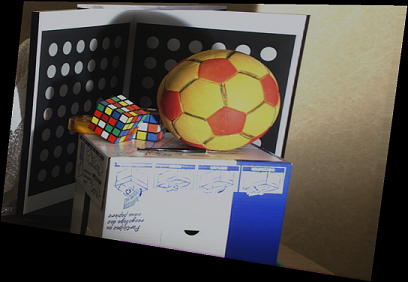} ~
\includegraphics[width=0.475\textwidth]{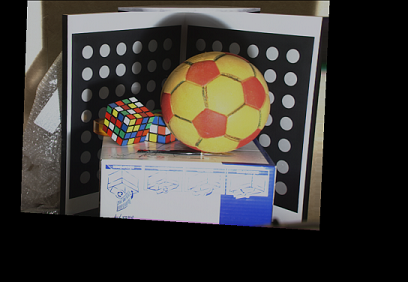} 
\caption{Rectification results with ``Fusiello'' \cite{fusiello2008quasi} method.}
\label{subfig:proposedMethod}
\end{subfigure}

\begin{subfigure}[t]{0.475\textwidth}
\includegraphics[width=0.475\textwidth]{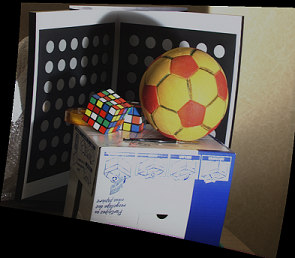} ~
\includegraphics[width=0.475\textwidth]{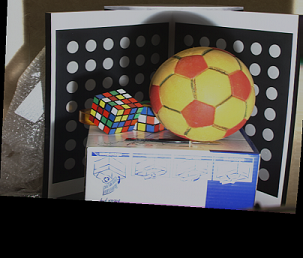} 
\caption{Rectification results with proposed method.}
\label{subfig:proposedMethod}
\end{subfigure}
\caption{An example of rectification result on SYNTIM dataset with different approaches.}
\label{fig:Rec_eva_per_exa}
\end{figure}                     
\subsection{Implementation time \label{Implementation_Time}}
The proposed fisheye stereo rectification approach is realized on a standard desktop (Intel 8 Cores i7) with Matlab R2016b processing environment. For Sony camera with resolution of 1920 $\times$ 1280, the whole rectification process takes about 15 seconds, in which about 6 seconds is spent on relative camera pose estimation process, 3 seconds for rectification model optimization and the final image rectification costs about 6 seconds. For the real application, the computation time can be significantly decreased by C/C++ implementation and GPU acceleration. Furthermore, for the real fisheye stereo configuration, the rectification can easily reach real-time implementation because both the relative camera pose and rectification model can be computed off-line and finally the rectification process can be easily realized by a look-up table strategy.   

\section{Closing Remarks and Future Works\label{sec:ClosingRemarks}}
Before reading this paper, the readers may have wondered the following question: is it not that stereo rectification is already a well-understood and solved old problem, and whether or not it deserves any further investigation? In this paper, we have given an affirmative answer, and we show that the conventional understanding of stereo rectification problem gives only one aspect of the problem. We reveal that, the admissible rectifying transformation can be a richer family than what we knew before. Based on this finding, we have develop an easy approach for fisheye and perspective image rectification and the experiment results show its effectiveness and robustness. Currently, the proposed method can only handle the case when both the epipoles are outside the images. In future, more general motions, e.g., forward motion, will be taken into consideration. Furthermore, we also propose to find a more general way of implementing the transformation matrix for fisheye stereo rectification.

\bibliographystyle{plainnat}
\bibliography{Rectification_Reference}

\begin{thebibliography}{28}
\providecommand{\natexlab}[1]{#1}
\providecommand{\url}[1]{\texttt{#1}}
\expandafter\ifx\csname urlstyle\endcsname\relax
  \providecommand{\doi}[1]{doi: #1}\else
  \providecommand{\doi}{doi: \begingroup \urlstyle{rm}\Url}\fi

\bibitem[Abraham and F{\"o}rstner(2005)]{abraham2005fish}
Steffen Abraham and Wolfgang F{\"o}rstner.
\newblock Fish-eye-stereo calibration and epipolar rectification.
\newblock \emph{Journal of photogrammetry and remote sensing}, 59\penalty0
  (5):\penalty0 278--288, 2005.

\bibitem[Andrilli and Hecker(2009)]{andrilli2010elementary}
Stephen Andrilli and David Hecker.
\newblock \emph{Elementary Linear Algebra}.
\newblock Academic Press, 2009.

\bibitem[Arican and Frossard(2007)]{arican2007dense}
Zafer Arican and Pascal Frossard.
\newblock Dense disparity estimation from omnidirectional images.
\newblock In \emph{Advanced Video and Signal Based Surveillance, Conference
  on}, pages 399--404. IEEE, 2007.

\bibitem[Bay et~al.(2006)Bay, Tuytelaars, and Van~Gool]{bay2006surf}
Herbert Bay, Tinne Tuytelaars, and Luc Van~Gool.
\newblock Surf: Speeded up robust features.
\newblock In \emph{European conference on computer vision}, pages 404--417.
  Springer, 2006.

\bibitem[Bouguet(2004)]{bouguet2004camera}
Jean-Yves Bouguet.
\newblock Camera calibration toolbox for matlab.
\newblock 2004.

\bibitem[Byrd et~al.(1999)Byrd, Hribar, and Nocedal]{byrd1999interior}
Richard~H Byrd, Mary~E Hribar, and Jorge Nocedal.
\newblock An interior point algorithm for large-scale nonlinear programming.
\newblock \emph{SIAM Journal on Optimization}, 9\penalty0 (4):\penalty0
  877--900, 1999.

\bibitem[Caruso et~al.(2015)Caruso, Engel, and Cremers]{caruso2015large}
David Caruso, Jakob Engel, and Daniel Cremers.
\newblock Large-scale direct slam for omnidirectional cameras.
\newblock In \emph{Intelligent Robots and Systems, International Conference
  on}, pages 141--148. IEEE, 2015.

\bibitem[Cignoni et~al.(2008)Cignoni, Callieri, Corsini, Dellepiane, Ganovelli,
  and Ranzuglia]{cignoni2008meshlab}
Paolo Cignoni, Marco Callieri, Massimiliano Corsini, Matteo Dellepiane, Fabio
  Ganovelli, and Guido Ranzuglia.
\newblock Meshlab: an open-source mesh processing tool.
\newblock In \emph{Eurographics Italian Chapter Conference}, volume 2008, pages
  129--136, 2008.

\bibitem[Fujiki et~al.(2007)Fujiki, Torii, and Akaho]{fujiki2007epipolar}
Jun Fujiki, Akihiko Torii, and Shotaro Akaho.
\newblock Epipolar geometry via rectification of spherical images.
\newblock In \emph{International Conference on Computer Vision/Computer
  Graphics Collaboration Techniques and Applications}, pages 461--471.
  Springer, 2007.

\bibitem[Fusiello and Irsara(2008)]{fusiello2008quasi}
Andrea Fusiello and Luca Irsara.
\newblock Quasi-euclidean uncalibrated epipolar rectification.
\newblock In \emph{Pattern Recognition, 19th International Conference on},
  pages 1--4. IEEE, 2008.

\bibitem[Gallup et~al.(2007)Gallup, Frahm, Mordohai, Yang, and
  Pollefeys]{gallup2007real}
David Gallup, Jan-Michael Frahm, Philippos Mordohai, Qingxiong Yang, and Marc
  Pollefeys.
\newblock Real-time plane-sweeping stereo with multiple sweeping directions.
\newblock In \emph{Computer Vision and Pattern Recognition, 2007. CVPR'07. IEEE
  Conference on}, pages 1--8. IEEE, 2007.

\bibitem[Geyer and Daniilidis(2003)]{geyer2003conformal}
Christopher Geyer and Kostas Daniilidis.
\newblock Conformal rectification of omnidirectional stereo pairs.
\newblock In \emph{Computer Vision and Pattern Recognition Workshop,}, pages
  73--73. IEEE, 2003.

\bibitem[Gluckman and Nayar(2001)]{gluckman2001rectifying}
Joshua Gluckman and Shree~K Nayar.
\newblock Rectifying transformations that minimize resampling effects.
\newblock In \emph{Computer Vision and Pattern Recognition, IEEE Conference
  on}, pages I--111, 2001.

\bibitem[H{\"a}ne et~al.(2014)H{\"a}ne, Heng, Lee, Sizov, and
  Pollefeys]{hane2014real}
Christian H{\"a}ne, Lionel Heng, Gim~Hee Lee, Alexey Sizov, and Marc Pollefeys.
\newblock Real-time direct dense matching on fisheye images using
  plane-sweeping stereo.
\newblock In \emph{International Conference on 3D Vision}, volume~1, pages
  57--64. IEEE, 2014.

\bibitem[Hartley(1999)]{hartley1999theory}
Richard~I Hartley.
\newblock Theory and practice of projective rectification.
\newblock \emph{International Journal of Computer Vision}, 35\penalty0
  (2):\penalty0 115--127, 1999.

\bibitem[Heller and Pajdla(2009)]{heller2009stereographic}
Jan Heller and Tomas Pajdla.
\newblock Stereographic rectification of omnidirectional stereo pairs.
\newblock In \emph{Computer Vision and Pattern Recognition, 2009. CVPR 2009.
  IEEE Conference on}, pages 1414--1421. IEEE, 2009.

\bibitem[Kannala and Brandt(2006)]{kannala2006generic}
Juho Kannala and Sami~S Brandt.
\newblock A generic camera model and calibration method for conventional,
  wide-angle, and fish-eye lenses.
\newblock \emph{Pattern Analysis and Machine Intelligence, IEEE Transactions
  on}, 28\penalty0 (8):\penalty0 1335--1340, 2006.

\bibitem[Ko et~al.(2016)Ko, Shim, Choi, and Kuo]{1603.09462}
Hyunsuk Ko, Han~Suk Shim, Ouk Choi, and C.~C.~Jay Kuo.
\newblock Robust uncalibrated stereo rectification with constrained geometric
  distortions (usr-cgd), 2016.

\bibitem[Li and Hartley(2006)]{li2006five}
Hongdong Li and Richard Hartley.
\newblock Five-point motion estimation made easy.
\newblock In \emph{Pattern Recognition, International Conference on}, volume~1,
  pages 630--633. IEEE, 2006.

\bibitem[Loop and Zhang(1999)]{loop1999computing}
Charles Loop and Zhengyou Zhang.
\newblock Computing rectifying homographies for stereo vision.
\newblock In \emph{Computer Vision and Pattern Recognition}, volume~1. IEEE,
  1999.

\bibitem[Monasse et~al.(2010)Monasse, Morel, and Tang]{monasse2010three}
Pascal Monasse, Jean-Michel Morel, and Zhongwei Tang.
\newblock Three-step image rectification.
\newblock In \emph{British Machine Vision Conference}, pages 89.1--10. BMVA
  Press, 2010.

\bibitem[Palander and Brandt(2008)]{palander2008epipolar}
Kimmo Palander and Sami~S Brandt.
\newblock Epipolar geometry and log-polar transform in wide baseline stereo
  matching.
\newblock In \emph{Pattern Recognition. International Conference on}, pages
  1--4. IEEE, 2008.

\bibitem[Pollefeys et~al.(1999)Pollefeys, Koch, and
  Van~Gool]{pollefeys1999simple}
Marc Pollefeys, Reinhard Koch, and Luc Van~Gool.
\newblock A simple and efficient rectification method for general motion.
\newblock In \emph{Computer Vision, International Conference on}, pages
  496--501. IEEE, 1999.

\bibitem[Scaramuzza et~al.(2006)Scaramuzza, Martinelli, and
  Siegwart]{scaramuzza2006toolbox}
Davide Scaramuzza, Agostino Martinelli, and Roland Siegwart.
\newblock A toolbox for easily calibrating omnidirectional cameras.
\newblock In \emph{Intelligent Robots and Systems, International Conference
  on}, pages 5695--5701, 2006.

\bibitem[Yang et~al.()Yang, Li, and Jia]{yang2014optimal}
Jiaolong Yang, Hongdong Li, and Yunde Jia.
\newblock Optimal essential matrix estimation via inlier-set maximization.
\newblock In \emph{ECCV 2014}, pages 111--126. Springer.

\bibitem[Zhang et~al.(2015{\natexlab{a}})Zhang, Li, Cheng, Cai, Chao, and
  Rui]{zhang2015meshstereo}
Chi Zhang, Zhiwei Li, Yanhua Cheng, Rui Cai, Hongyang Chao, and Yong Rui.
\newblock Meshstereo: A global stereo model with mesh alignment regularization
  for view interpolation.
\newblock In \emph{IEEE International Conference on Computer Vision}, pages
  2057--2065, 2015{\natexlab{a}}.

\bibitem[Zhang et~al.(2015{\natexlab{b}})Zhang, Yao, Xia, Li, Zhang, and
  Liu]{zhang2015line}
Mi~Zhang, Jian Yao, Menghan Xia, Kai Li, Yi~Zhang, and Yaping Liu.
\newblock Line-based multi-label energy optimization for fisheye image
  rectification and calibration.
\newblock In \emph{IEEE Conference on Computer Vision and Pattern Recognition},
  pages 4137--4145. IEEE, 2015{\natexlab{b}}.

\bibitem[Zhang et~al.(2016)Zhang, Rebecq, Forster, and
  Scaramuzza]{zhangbenefit}
Zichao Zhang, Henri Rebecq, Christian Forster, and Davide Scaramuzza.
\newblock Benefit of large field-of-view cameras for visual odometry.
\newblock In \emph{International Conference on Robotics and Automation}, pages
  801--808. IEEE, 2016.

\end{thebibliography}
\end{document}